\documentclass[11pt,authoryear]{article}
\input{preamble}

\begin{document}
\title{Solar: $L_0$ solution path averaging for fast and accurate variable selection\\ in high-dimensional data}

\author{Ning Xu\thanks{n.xu@sydney.edu.au} \and Timothy C.G. Fisher\thanks{tim.fisher@sydney.edu.au}}
\date{School of Economics, University of Sydney, Australia}

\maketitle

\begin{abstract}
We propose a new algorithm for variable selection in high dimensional and large scale data, \emph{subsample-ordered least-angle regression (solar)}, and its coordinate descent generalization, \emph{solar-cd}. Solar re-constructs lasso paths using the $L_0$ norm and averages the resulting solution paths across subsamples. Path averaging retains the ranking information of the informative variables while averaging out sensitivity to high dimensionality, improving variable selection stability, efficiency, and accuracy. Using the same numerical optimzers as lasso does, solar can be generalized to different lasso variants. We prove that: (i) with a high probability, path averaging perfectly separates informative variables from redundant variables on the average $L_0$ path; (ii) solar variable selection is consistent and accurate; and (iii) the probability that solar omits weak signals is controllable for finite sample size. Using simulations, examples, and real-world data, we demonstrate the following advantages of solar: (i) solar yields, with less than $1/3$ of the lasso computation load, substantial improvements over lasso in terms of the sparsity (64-84\% reduction in redundant variable selection) and accuracy of variable selection; (ii) compared with the lasso safe/strong rule and variable screening, solar largely avoids selection of redundant variables and rejection of informative variables in the presence of complicated dependence structures and harsh settings of the irrepresentable condition; (iii) the sparsity and stability of solar conserves residual degrees of freedom for data-splitting hypothesis testing, improving the accuracy of post-selection inference on weak signals with limited $n$; (iv) replacing lasso with solar in bootstrap selection (e.g., bolasso or stability selection) produces a multi-layer variable ranking scheme that improves selection sparsity and ranking accuracy with the computation load of only one lasso realization; and (v) given the computation resources, solar bootstrap selection is substantially faster (98\% lower computation time) than the theoretical maximum speedup for parallelized bootstrap lasso (confirmed by Amdahl's law). The efficiency of bootstrap solar makes cross validation computationally affordable for optimizing the bootstrap selection threshold even in large scale and high dimensional data.
\end{abstract}


\spacingset{1.4}


\pagenumbering{arabic}

\section{Introduction}

Recent innovations to lasso-type algorithms \citep{efronall04, friedman2007pathwise, friedman2010regularization} have largely addressed selection of redundant variables, rejection of informative variables, and poor performance under high multicollinearity in high dimensional ($p>n$) and large scale data (large $p$ and large $n$). However, in alleviating old problems, the innovations have revealed new challenges.

Bootstrap variable selection [e.g., \citet{bach2008bolasso}, \citet{meinshausen2010stability}, \citet{wang2011random}, and \citet{mameli2017estimating}] markedly improves variable selection sparsity and inference accuracy, yet it requires repeating lasso and its variants (often with cross-validation) on hundreds of bootstrap subsamples to average the variable selection results or the inference results. \citet{xu2012asymptotic} and Sections~\ref{subsection:suml1} and \ref{subsection:comp} below illustrate that bootstrap selection methods exponentially increase computation load, limiting applicability in large scale data such as DNA sequencing, image recognition, fMRI and MRI data of the neuroimaging, and natural language processing (where both $p$ and $n$ are often over $10,000$). More seriously, choosing the bootstrap variable selection threshold, which is often set based on field experience or simulations, remains an unsolved issue. \citet{bach2008bolasso} and \citet[Figure~2]{huang2014stat} illustrate that a pre-defined threshold may omit informative variables (low power) and select redundant variables (high false discovery rate) in both high and low dimensions.

One strategy to improve lasso selection sparsity without increasing computation burden is to use a post-selection rule to screen variables selected by lasso. Post-lasso selection rules [e.g., the `safe rule' \citep{ghaoui2010safe} and the `strong rule' \citep{tibshirani2012strong}] are capable of reducing the number of variables to enhance computational efficiency in lasso. However, recent research \citep{wang2014safe, zeng2017efficient} and Section~\ref{section:example} suggest both rules may be prone to rejecting informative variables, selecting redundant variables, or proposing repeated modifications (e.g., rejecting a variable in an early round and adding it back in a later round).

Data-splitting hypothesis tests are another way to screen variables selected by lasso \citep{wasserman2009high, meinshausen2009p,romano2019multiple, diciccio2020exact}. The original data are divided into two: one part for variable selection, the other part for testing. However, to improve test power, data splitting is repeated on each bootstrap subsample, raising similar computational concerns as bootstrapping variable selection \citep{bach2008bolasso}. \citet{diciccio2020exact} also argue that because data splitting reserves some of the data for variable selection, it reduces the degrees of freedom for testing on the remaining data, presenting a challenge to detect weak signals when sample size is limited.

Specifically designed to address the challenges of high dimensional data, the variable screening algorithm \citep{fan2008sure, hall2009using,hall2009usingb, li2012robust, li2012feature} ranks the absolute values of unconditional correlations between each covariate and the response variable, selecting only the top-ranked variables. However, \citet{fan2008sure}, \citet{barut2016conditional}, and Section~\ref{section:example} below show that variable screening also suffers from selection of redundant variables and rejection of informative variables when the dependence structures are complicated.

According to \citet{friedman2001elements, weisberg04}, forward selection was historically dismissed in high-dimensional spaces due to inefficiency and sensitivity to sampling randomness, multicollinearity, noise, and outliers due to the iterative refitting of the residual. \citet{tibshirani2015general} illustrates through simulation that (i) forward selection may produce similar generalization errors to lasso-type estimators for fitted models and (ii) that forward selection is computationally competitive to lasso in different applications (image de-noising, matrix completion, etc.). However, \citet{tibshirani2015general} does not suggest any solution to a range of issues for forward selection or lasso (solved by lars), including instability of variable selection, selection of redundant variables, lack of robustness to the irrepresentable condition and complicated dependence structures, or sensitivity to sampling randomness, multicollinearity, noise, and outliers. Moreover, \citet{tibshirani2015general} demonstrates the computation speedup through comparison without providing any rigorous analysis.

\subsection{Main results}

To address issues above, we propose a new forward selection algorithm, \emph{subsample-ordered least-angle regression (solar)}, and its coordinate-descent generalization, \emph{solar-cd}.

Solar re-constructs lasso paths using the $L_0$ norm and averages the resulting solution paths across subsamples. Path averaging retains the ranking information of the informative variables while averaging out sensitivity to high dimensionality, improving variable selection stability, efficiency, and accuracy. Using the same numerical optimizers as lasso does, solar can be easily generalized to many lasso variants. Under the \citet{zhang09} general framework of forward selection, we prove that: (i) with a high probability, path averaging perfectly separates informative variables from redundant variables on the average $L_0$ path; (ii) solar variable selection is consistent and accurate under the general framework of forward selection; and (iii) the probability that solar omits weak signals is controllable for finite sample size.

Using simulations, examples, and real-world data, we demonstrate the following advantages of solar: (i) solar yields, with less than $1/3$ of the lasso computation load, substantial improvements over lasso in terms of the sparsity (64-84\% reduction in redundant variable selection), stability, and accuracy of variable selection; (ii) compared with the lasso safe/strong rule and variable screening, solar largely avoids selection of redundant variables and rejection of informative variables in the presence of complicated dependence structures and harsh settings of the irrepresentable condition; (iii) the sparsity of solar conserves residual degrees of freedom for data-splitting hypothesis testing, improving the efficiency and accuracy of post-selection inference for weak signals with limited $n$; (iv) replacing lasso with solar in subsampling selection (e.g., the bootstrap lasso or stability selection) produces a multi-layer variable ranking scheme that improves selection sparsity and ranking accuracy with the computation load of only one lasso realization; and (v) Given the computation resources, solar bootstrap is substantially faster (98\% lower computation time) than the theoretical maximum speedup for parallelized bootstrap lasso (confirmed by Amdahl's law). The efficiency of bootstrap solar makes cross validation computationally affordable for optimizing the bootstrap selection threshold even in large scale and high dimensional data. We provide a parallel computing package for solar (\texttt{solarpy}) that uses a Python interface and an Intel MKL Fortran/C++ compiler in a supplementary file and dedicated \href{https://github.com/isaac2math/solarpy}{Github page}.

The paper is organized as follows. In Section~\ref{section:algo}, we introduce the solar algorithm, show the theoretical properties of path avergaing and solar, explain the coordinate descent generalization of solar, and discuss generalizations of solar to variants of lasso. In Section~\ref{section:adv}, we use examples to demonstrate the advantages of solar over lasso, the safe/strong rules, and variable screening . In Section~\ref{section:comp}, we use simulations to demonstrate the advantages of solar over lasso-type algorithms in terms of variable selection sparsity, accuracy, and computation load. In Section~\ref{section:application}, we use real-world data to show that the improvements from solar are feasible in the presence complicated dependence structures, while lasso and elastic net [the lasso variant alleged \citep{zou2005regularization, jia2010model} to have the best selection accuracy and sparsity under multicollinearity] completely lose sparsity. The proofs of the properties of solar are in Supplementary Material~A. The \texttt{solarpy} code and raw simulation results are in Supplementary Material~B.


\section{The Solar algorithm \label{section:algo}}

The key to solar lies in the parameterization of the solution path. For any forward selection method, \citet[Theorem~2]{zhang09} shows that the earlier a variable enters the solution path, the more likely it is to be informative. Thus, an accurate and stable ordering of variables in the solution path may help to identify the informative variables. Since we focus on accuracy, the only relevant feature of the regression coefficients in the solution path is whether $\beta_i = 0$ at each stage. Thus, solar parameterizes the lasso path (or, more generally, any forward selection path) using the $L_0$ norm.
\begin{definition}[$L_0$ solution path]
  Define the \textbf{$L_0$ solution path} on $\left( Y, X \right)$ to be the order that least angle regression includes variables across all stages. For example, if the least angle regression includes $\mathbf{x}_3$ at stage 1, $\mathbf{x}_2$ at stage 2 and $\mathbf{x}_1$ at stage 3, the corresponding $L_0$ path is the ordered set $\left\{ \mathbf{x}_3, \mathbf{x}_2, \mathbf{x}_1 \right\}$.
  \label{def:solution_path}
\end{definition}

\subsection{Solar optimized by least angle regression}

The solar algorithm involves two steps: \emph{parameterizing and averaging $L_0$ paths} and \emph{selecting variables on the average $L_0$ path}.

\subsubsection{$L_0$ path parameterizing and averaging}

The solution path is the foundation of variable (feature) selection in $L_p$-regularized linear modelling. The first step in solar is to improve the robustness of the solution path to high dimensional issues such as multicollinearity, complicated dependence structures, noise, weak signals, etc. In particular, there are two major concerns.
\begin{itemize}
  \item Computation efficiency: computation load is a central concern in subsampling-based model averaging. Because bootstrap methods (e.g., bootstrap lasso) require hundreds of lasso repetitions to average out variable selection issues in high dimensions, they are computationally expensive with large $n$ and large $p$. Thus, improving selection performance and reducing the number of repetitions would go a long way to reducing computation load.
  \item Averaging efficiency: the $L_1$ lasso solution path (solved by lars) is essentially a piecewise linear function $\beta = g(\lambda)$, which is easy to average. By contrast, it is not obvious how to average the $L_0$ path because it is an ordered set of rankings. If we average the ranks each $\mathbf{x}_i$ enters the path in large $p$ problems, a weak signal (i.e., an $\mathbf{x}_i$ with a small but non-zero $\left\Vert \beta_i \right\Vert_1$ in the population) may occasionally be ranked at a later stage, returning a large stage value, and exerting outlier influence on the stage value averaging. In other words, to accurately average solution paths using as few subsamples as possible, we need a parameterization method for the $L_0$ path that is more robust to outliers in ranking $\mathbf{x}_i$.
\end{itemize}

\noindent
Our solution to these concerns is the $\widehat{q}$ \emph{method}, summarized in Algorithm~\ref{algo:APE-lar} and illustrated in Figure~\ref{fig:q_demo}. For solution path averaging, instead of using the stage value each $\mathbf{x}_i$ enters the solution path, which varies in the range $\left[0, +\infty \right)$, Algorithm~\ref{algo:APE-lar} uses $\widehat{q}^{\,k}_i$, which normalizes the stage value into the range $\left[ 0,1 \right]$. The concentration inequalities for empirical processes show that averaging $\widehat{q}^{\,k}_i$ across subsamples converges much faster and is much more stable than averaging the stage values of the $\mathbf{x}_i$.

\smallskip
\begin{algorithm}[ht]

  \SetKwData{Left}{left}\SetKwData{This}{this}\SetKwData{Up}{up}
  \SetKwFunction{Union}{Union}\SetKwFunction{FindCompress}{FindCompress}
  \SetKwInOut{Input}{input}\SetKwInOut{Output}{output}

  \smallskip
  \Input{$\left( Y, X \right)$.}

  divide the original sample equally into $K$ folds and generate $K$ subsamples $\left\{ \left( Y^k, X^k \right) \right\}^{K}_{k=1}$ by removing one fold in turn from $\left( Y, X \right)$\;

  set $\widetilde{p} = \min\left\{ n\left(K-1\right)/K, p \right\}$\;

  \For{ k := 1 to K, stepsize = 1 \nllabel{outer_averaging_start} }{

    run an unrestricted least angle regression (or any forward selection algorithm) on $\left( Y^k, X^k \right)$ and record the order of variable inclusion at each stage\;
    \nllabel{inner_averaging_start}

    define $\widehat{q}^k = \mathbf{0} \in \mathbb{R}^p$\;

    $\forall i,l \in \mathbf{N}^+$, if $\mathbf{x}_i$ is included at stage $l$ and excluded at $l-1$, set $\widehat{q}^k_i= (\widetilde{p} + 1 - l) / \widetilde{p}$, where $\widehat{q}^k_i$ is the $i$\textsuperscript{th} entry of $\widehat{q}^k$\;
    \nllabel{inner_averaging_end}

    }

  $\widehat{q} := \frac{1}{K} \sum_{k=1}^{K} \widehat{q}^k$\; \nllabel{outer_averaging_end}

  \Return $\widehat{q}$

\caption{$\widehat{q}$ method: parameterizing and averaging $L_0$ solution paths} \label{algo:APE-lar}

\end{algorithm}


\begin{figure}[ht]
  \centering
  \includegraphics[width=0.66\paperwidth]{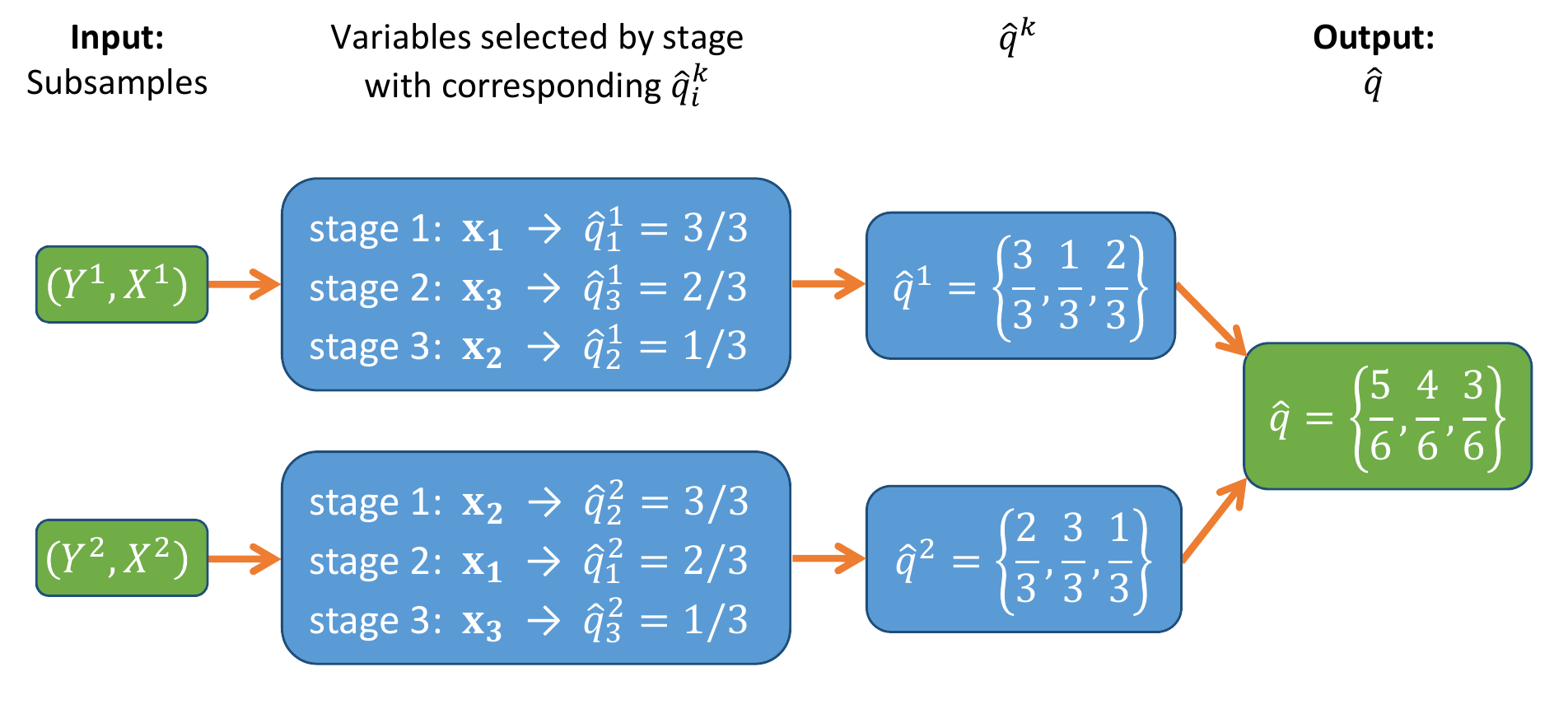}
  \caption{Computation of $\widehat{q}$ on 2 subsamples by least angle regression.}
  \label{fig:q_demo}
\end{figure}

After the subsamples are created, lines~\ref{inner_averaging_start}-\ref{inner_averaging_end} of Algorithm~\ref{algo:APE-lar} compute $\widehat{q}^k$, which summarizes the order that least angle regression includes each $\mathbf{x}_i$ across all stages (see Figure~\ref{fig:q_demo}). The unrestricted least angle regression ranks variables by the stage they enter the solution path. As shown in line~\ref{inner_averaging_end} of Algorithm~\ref{algo:APE-lar} and Figure~\ref{fig:q_demo}, variables included at earlier stages have larger $\widehat{q}^k_i$ values: the first variable included is assigned $1$, the last is assigned $1/\widetilde{p}$, while the rejected variables are assigned $0$ (which occurs only when $p > n$). Thus, the $L_0$ solution path is obtained by ranking the $\mathbf{x}_i$ according to their $\widehat{q}^k_i$ values.

\citet[Theorem 2]{zhang09} implies that, on average, variables with the largest $\widehat{q}^k_i$ values are more likely to be informative. The $\widehat{q}^k_i$ may be sensitive in high-dimensional spaces to multicollinearity, sampling randomness, and noise. In these circumstances, a redundant variable may be included at an early stage in some $\left( Y^k, X^k \right)$ subsample. Algorithm~\ref{algo:APE-lar} reduces the impact of sensitivity in the $\widehat{q}^k_i$ by computing $\widehat{q} := \frac{1}{K} \sum_{k=1}^{K} \widehat{q}^k$ and ranking the $\mathbf{x}_i$ according to $\widehat{q}_i$ (the $i$\textsuperscript{th} entry in $\widehat{q}$), to arrive at the average $L_0$ solution path. The average $L_0$ solution path is formally defined as follows.


\begin{definition}[average $L_0$ solution path]
  Define the \textbf{average $L_0$ solution path} of least angle regression on $\left\{ \left( Y^k, X^k \right) \right\}_{k=1}^{K}$ to be the (decreasing) rank order of the $\mathbf{x}_i$ variables based on their corresponding $\widehat{q}_i$ values. For example, in Figure~\ref{fig:q_demo}, the $\widehat{q}_i$ for $\mathbf{x}_1$, $\mathbf{x}_2$ and $\mathbf{x}_3$ are, respectively, $\widehat{q}_1 = 5/6$, $\widehat{q}_2 = 4/6$ and $\widehat{q}_3 = 3/6$. Thus, the average $L_0$ solution path may be represented as an ordered set $\left\{ \mathbf{x}_1, \mathbf{x}_2, \mathbf{x}_3 \right\}$.
  \label{def:L_0_solution_path}
\end{definition}


To justify theoretically the $\widehat{q}$ method, we use the \citet{zhang09} framework to derive the theoretical properties of path averaging (see Appendix~A).

\begin{itemize}
  \item Under the \citet{zhang09} conditions, Lemma~\ref{lemma:1} shows that, with a high probability, using $\widehat{q}^{\,k}_i$ ranking for variable selection on $\left( Y^k, X^k \right)$ generates the same theoretical results as the \citet{zhang09} forward selection method.
  \item Under a similar stopping condition to \citet{zhang09}, Lemma~\ref{lemma:2} shows that, with a high probability, there exists a threshold $c^k$ for the $L_0$ path on $\left( Y^k, X^k \right)$ such that $\widehat{q}^{\,k}_i \geqslant c^k$ for informative $\mathbf{x}_i$ and $\widehat{q}^k_i < c^k$ for redundant $\mathbf{x}_i$
  \item Using Lemma~\ref{lemma:2}, Lemma~\ref{lemma:3} shows that there exists a threshold $c = \sum_{i=k}^{K} c^k/K$ for the average $L_0$ path such that $\widehat{q}_i \geqslant c$ for informative $\mathbf{x}_i$ and $\widehat{q}_i < c$ for redundant $\mathbf{x}_i$ with large probability.
\end{itemize}

\subsubsection{Variable selection on the average $L_0$ path}

The solar algorithm is constructed on the aveage $L_0$ path and summarized in Algorithm~\ref{algo:solar}. We present solar under the generic framework of forward regression and can easily adapt it to least angle regression, forward or backward selection algorithms.


\begin{algorithm}[ht]

  \SetKwData{Left}{left}\SetKwData{This}{this}\SetKwData{Up}{up}
  \SetKwFunction{Union}{Union}\SetKwFunction{FindCompress}{FindCompress}
  \SetKwInOut{Input}{input}\SetKwInOut{Output}{output}

  \smallskip
  Randomly select 20\% of the sample points as the validation set; denote the remaining points as the training set\;

  Estimate $\widehat{q}$ using Algorithm~\ref{algo:APE-lar} on the training set and compute $Q(c) = \left\{ \mathbf{x}_j \; \vert \; \widehat{q}_j \geqslant c, \forall j\right\}$ for all $c \in \left\{ 1, 0.98, \ldots, 0.02, 0 \right\}.$

  Run an OLS regression of each $Q(c)$ on $Y$ using the training set and find $c^*$, the value of $c$ that minimizes the validation error\;

  Compute the OLS coefficients of $Q(c^*)$ on $Y$ using the whole sample.

  \caption{Subsample-ordered least-angle regression (solar) \label{algo:solar}}
\end{algorithm}


In Algorithm~\ref{algo:solar}, variables are included into forward regression according to their rank order in the average $L_0$ solution path, represented by $\left\{ Q(c) \vert c = 1, 0.98, \ldots, 0\right\}$ in Algorithm~\ref{algo:solar}. We use $\widehat{q}$ from Algorithm~\ref{algo:APE-lar} to generate a list of variables $Q \left( c \right) = \left\{ \mathbf{x}_j \; \vert \; \widehat{q}_j \geqslant c, \forall j \leqslant p \right\}$. For any $c_1 > c_2$, $Q\left(c_1\right) \subset Q\left(c_2\right)$, implying a sequence of nested sets $\left\{ Q(c) \vert c = 1, 0.98, \ldots, 0\right\}$. Each $c$ denotes a stage of forward regression. For a given value of $c$, $Q(c)$ denotes the set of variables with $\left\Vert \beta_i \right\Vert_0=1$ on average and $Q(c) - Q(c - 0.02)$ is the set of variables with $\left\Vert \beta_i \right\Vert_0$ just turning to $1$ at $c$. Therefore, $\left\{ Q(c) \vert c = 1, 0.98, \ldots, 0\right\}$ is the average $L_0$ solution path of Definition~\ref{def:L_0_solution_path}. Variables that are more likely to be informative have larger $c$ values in $Q(c)$ and will be selected first by the solar algorithm.

Using the \citet{zhang09} framework and Lemmas~\ref{lemma:2} and~\ref{lemma:3}, we derive the following theoretical results for variable selection (see Appendix~A).

\begin{itemize}
  \item Theorem~\ref{thm:1} shows that solar variable selection is $L_0$ consistent under similar sparse eigenvalue and irrepresentable conditions as have been used to prove  lasso consistency.
  \item Under similar assumptions to \citet{zhang09}, Lemmas~\ref{lemma:4} and \ref{lemma:5} show that the number of omitted informative $\mathbf{x}_i$ and the probability of selecting at least one redundant $\mathbf{x}_i$ are restricted by sample size, the sparse eigenvalue condition, and the stopping condition.
\end{itemize}

The key difference between solar and the lasso-type estimators, and the source of the advantages of solar, is solution path averaging.
\begin{itemize}
  \item The difference between solar and lasso is that solar averages the solution path. Lasso and solar both use the solution path for variable selection. Lasso and its variants focus on optimizing the shrinkage parameter $\lambda$ (via cross validation), leaving aside concerns about the reliability of the lasso path in high dimensions. Optimizing $\lambda$ on an unreliable path renders variable selection difficult. By contrast, solar  prioritizes averaging the solution path, which not only averages out  path unreliability in high dimensions, but also ranks all the informative variables at the start of the average $L_0$ path (as shown in Lemma~\ref{lemma:2} and~\ref{lemma:3}). Hence, with a high probability, the variable selection algorithm needs only to analyze the variables at the start of the average $L_0$ path, making selection accurate and efficient.
  \item The difference between solar and lasso-related bootstrap selection (e.g., bolasso) is in how they average the variable selection algorithm. Given the $\lambda$ value (optimal or not), lasso-related bootstrap selection averages the selection \emph{results} across subsamples. Thus, bootstrap selection requires hundreds of repetitions to average out the instability and redundancy of lasso variable selection \citep{bach2008bolasso}. By contrast, solar averages solution \emph{paths}, which solves most of the lasso instability and redundancy issues, returning a more reliable path (the average $L_0$ path). Variable selection along a reliable path substantially reduces the likelihood that solar selects redundant variables or omits informative variables. As a result, solar-related bootstrap selection (e.g., bootstrap solar or solar stability selection) requires only 3-5 repetitions to outperform hundreds of lasso-related bootstrap repetitions (see Section~\ref{subsection:comp} for details).
\end{itemize}

\subsection{Solar optimized by coordinate descent}

The solar algorithm can easily be generalized to use coordinate descent. For lasso, least angle regression or coordinate descent generates the same solution path parameterized by the $\beta_i$ and the shrinkage parameter $\lambda$. Thus, to reprogram solar to use coordinate descent, we simply replace Algorithm~\ref{algo:APE-lar} with Algorithm~\ref{algo:APE-cd}, which records the order of variable selection along the coordinate descent solution path.

\smallskip
\begin{algorithm}[ht]

  \SetKwData{Left}{left}\SetKwData{This}{this}\SetKwData{Up}{up}
  \SetKwFunction{Union}{Union}\SetKwFunction{FindCompress}{FindCompress}
  \SetKwInOut{Input}{input}\SetKwInOut{Output}{output}

  \smallskip
  \Input{$\left( Y, X \right)$.}

  generate $K$ subsamples $\left\{ \left( Y^k, X^k \right) \right\}^{K}_{k=1}$ by randomly remove $1/K$ of observations in $\left( Y, X \right)$\;

  set $\widetilde{p} = \min\left\{ n_{\mathrm{sub}}, p \right\}$ \;

  \For{ k := 1 to K, stepsize = 1 \nllabel{outer_averaging_start3} }{

    denote $\lambda_s$ as the shrinkage parameter value that coordinate descent lasso selects $s$ variables, $\forall s \in \left[ 0, \widetilde{p}\right]$;

    run a pathwise coordinate descent for lasso on $\left( Y^k, X^k \right)$, $\forall \lambda \in \left\{\lambda_0, \lambda_1, \ldots, \lambda_{\widetilde{p}},\right\}$

    record the order of variable inclusion at each $\lambda \in \left\{\lambda_0, \lambda_1, \ldots, \lambda_{\widetilde{p}},\right\}$\;

    define $\widehat{q}^k = \mathbf{0} \in \mathbb{R}^p$\;

    $\forall i,s \in \mathbf{N}^+$, if $\mathbf{x}_i$ is included at $\lambda = \lambda_s$ and excluded at $\lambda_{s-1}$, set $\widehat{q}^k_i= (\widetilde{p} + 1 - s) / \widetilde{p}$, where $\widehat{q}^k_i$ is the $i$\textsuperscript{th} entry of $\widehat{q}^k$\;

  }

  $\widehat{q} := \frac{1}{K} \sum_{k=1}^{K} \widehat{q}^k$\; \nllabel{outer_averaging_end3}

  \Return $\widehat{q}$

\caption{average $L_0$ solution path estimation via coordinate descent \label{algo:APE-cd}}

\end{algorithm}

\begin{figure}[ht]
  \centering
  \includegraphics[width=0.66\paperwidth]{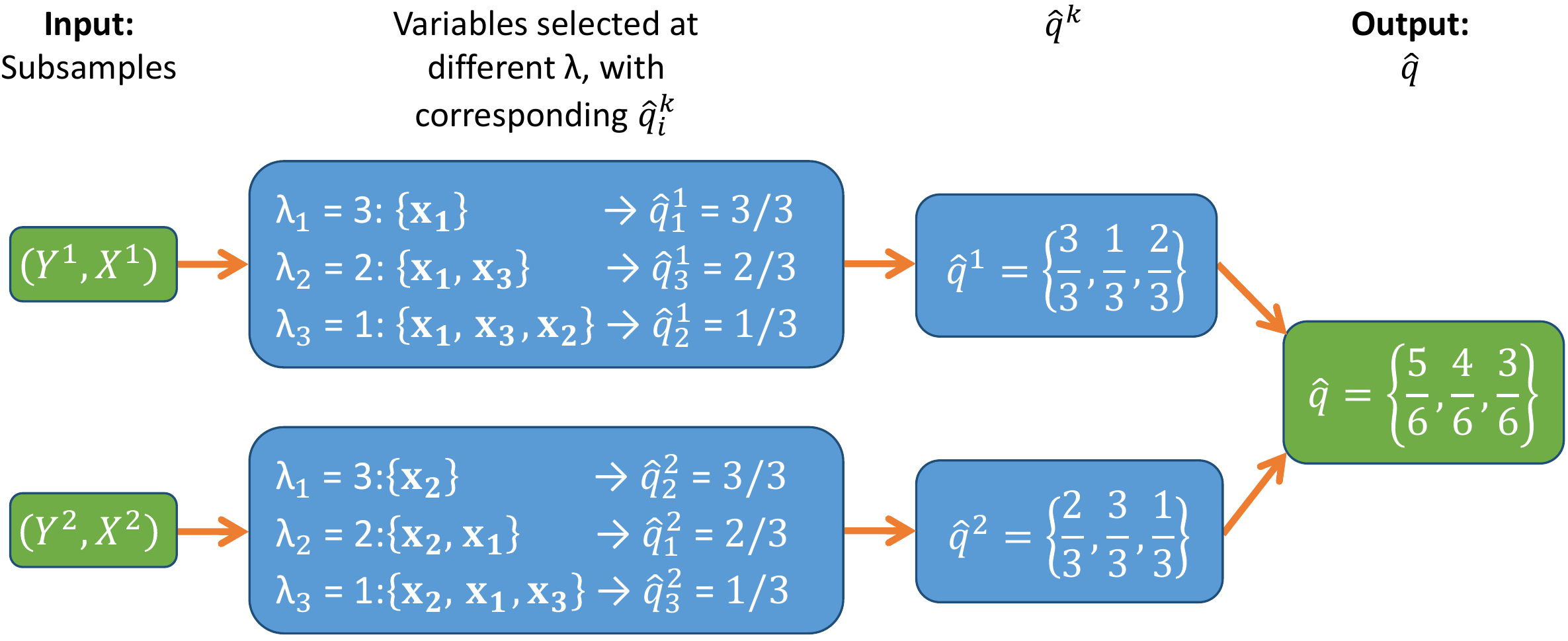}
  \caption{Computation of $\widehat{q}$ on 2 subsamples using coordinate descent.}
  \label{fig:q_demo_3}
\end{figure}

Algorithm~\ref{algo:APE-cd} serves the same purpose as Algorithm~\ref{algo:APE-lar}: to estimate the average $L_0$ path. Algorithm~\ref{algo:APE-cd} uses $\lambda$ to record the order that each variable enters the path. Consider the example in Figure~\ref{fig:q_demo_3}. To re-parameterize the solution path, we denote $\lambda_s$ to be the $\lambda$ value that coordinate descent lasso includes $s$ variables, $\forall s\in \left( 0, \min \left\{ n/2, p \right\} \right]$, giving a sequence of $\lambda$ for grid search. In each subsample $\left( Y^k, X^k \right)$, we train a standard pathwise coordinate descent for lasso, allowing $\lambda$ to increase stepwise within the grid $\left\{\lambda_1, \ldots, \lambda_{ \min \left\{ n/2, p \right\} } \right\}$, where $\lambda_1 \geqslant \ldots \geqslant \lambda_{ \min \left\{ n/2, p \right\} }$. In Figure~\ref{fig:q_demo_3}, when $\lambda \leqslant \lambda_3$ at subsample $\left( Y^1, X^1 \right)$, all three variables are selected in the solution path, implying that $\widehat{q}^1_i \geqslant 1/3$ for all variables. When $\lambda$ increases to $\lambda_2$, only $\{\mathbf{x}_3, \mathbf{x}_1\}$ survive the harsher shrinkage, implying that they should be ranked higher than $\mathbf{x}_2$. As a result, $\widehat{q}^1_1, \widehat{q}^1_3 \geqslant 2/3$ and $\widehat{q}^1_2 = 1/3$. When $\lambda$ reaches $\lambda_3$, only $\{\mathbf{x}_1\}$ remains, leaving $\widehat{q}^1_1 = 3/3$ and $\widehat{q}^1_3 = 2/3$. Applying the same method to each subsample produces the same $\widehat{q}$ as Algorithm~\ref{algo:APE-lar}.

\subsection{Comparison and generalization to lasso variants \label{subsection:variant}}
\label{subsec:variant}

Because solar is trained by least angle regression or coordinate descent, it can easily be extended to several lasso variants:

\begin{itemize}
  \item `Grouped solar' is invoked by forcing specific variables to be simultaneously selected into the solution path;
  \item `Adaptive solar' is obtained by weighting variable rankings in the average $L_0$ path according to their OLS coefficients;
  \item `Solar elastic net' or `fused solar' is derived by replacing the coordinate descent loss function in Algorithm~\ref{algo:APE-cd} with the $L_1$-$L_2$ loss
    \begin{equation}
      \left\Vert Y -X\beta \right\Vert_2^2 + \lambda^{(1)} \left\Vert \beta \right\Vert_1 + \lambda^{(2)} \left\Vert \beta \right\Vert_2^2
    \end{equation}
    or fused loss
    \begin{equation}
      \left\Vert Y -X\beta \right\Vert_2^2 + \lambda^{(1)} \left\Vert \beta \right\Vert_1 + \lambda^{(2)} \sum_{j=2}^{p} \left\vert \beta_j - \beta_{j-1} \right\vert_1.
    \end{equation}
\end{itemize}

Furthermore, many lasso enhancements (e.g., safe/strong rules, post-lasso hypothesis testing) may be applied to solar because they use the same optimization methods. Rather than competing with the lasso enhancements, solar supplements them by improving variable selection performance and computation speed in large-scale applications.

\section{Solar advantages over lasso variants, lasso rules, and variable screening \label{section:adv}}

In this section, we use a series of examples to demonstrate the advantages of the solar algorithm for post-selection hypothesis testing, in the presence of complicated dependence structures, and in terms of its robustness to the \emph{irrepresentable condition} (IRC).

\subsection{Post-selection hypothesis testing}

A major advantage of solar is its amenability to post-selection testing. Because the lasso tests \citep{lockhartall14, taylor2014exact} are based on forward regression, they may be adapted to solar. More interestingly, it is straightforward to adapt the data-splitting tests \citep{wasserman2009high,meinshausen2009p} to solar for weak signal detection. We illustrate this point using Example~1.

\smallskip
\noindent
\textbf{Example 1.} Consider the DGP
\begin{equation}
  Y = \mathbf{x}_0 + 2 \mathbf{x}_1 + 3 \mathbf{x}_2 + 4 \mathbf{x}_3 + 5 \mathbf{x}_4 + \sum_{j=5}^{p} 0 \cdot \mathbf{x}_j + e,
\end{equation}
where $\mathbf{x}_i$, $i=0,\dots,p$, are standard Gaussian variables with pairwise correlations of $0.5$, $e$ is a standard Gaussian noise term, and $p/n=100/100$.

Following \citet[Example~4.1]{romano2019multiple} and \citet{diciccio2020exact}, we conduct data-splitting tests by randomly separating the data into two portions of 50 observations. In the first round, one portion is used for solar or lasso selection and the other for testing. In the second round, the roles of the two portions are reversed. As a result, the p-values of any given variable are uncorrelated across the two rounds. Thus, we may apply Theorem~3.2 of \citet{romano2019multiple} and compute the average p-value across the two rounds to conduct a valid t-test for any selected covariate.

\citet{diciccio2020exact} stresses the importance of retaining residual degrees of freedom to ensure accurate tests. Because solar yields a more sparse and accurate variable selection than lasso does(Section~\ref{section:comp}), it conserves residual degrees of freedom, improving the reliability of post-selection p-values. Figure~\ref{fig:p_value_compare} plots the average p-values for the informative variables $\{\mathbf{x}_0,\ldots,\mathbf{x}_4\}$ from post-solar and post-lasso data-splitting tests using 100 repetitions. While the solar and lasso p-values are less than $0.05$ for the stronger signals $\{\mathbf{x}_1,\ldots,\mathbf{x}_4\}$, more than $25\%$ of the lasso p-values exceed $0.05$ for the weakest signal $\mathbf{x}_0$, implying non-trivial false non-rejection of $H_0$. By contrast, the solar p-value boxplot is very compact for $\mathbf{x}_0$, with only $5$ out of $100$ above $0.05$. Hence, solar p-values are more reliable for detecting weak signals with small $n$ and large $p$.

\begin{figure}[ht]
  \centering
  \includegraphics[width=0.7\paperwidth]{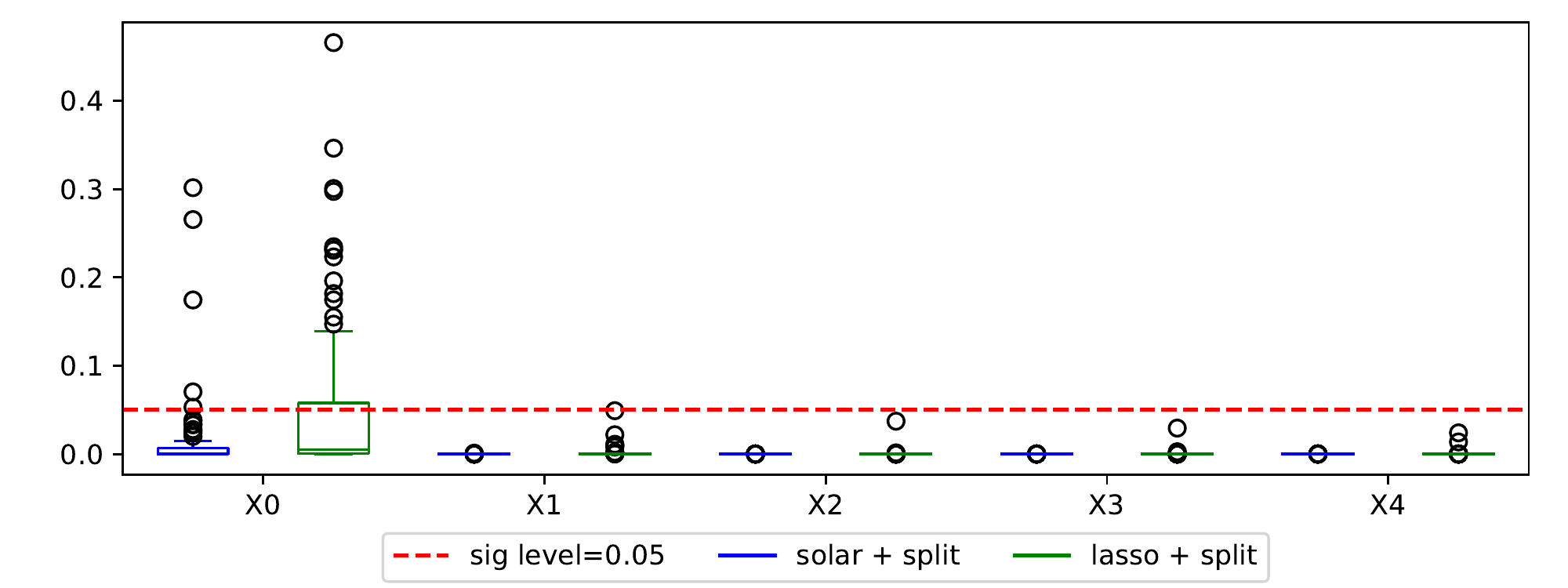}
  \caption{Average p-value boxplots for data-splitting t-tests with solar and lasso.}
  \label{fig:p_value_compare}
\end{figure}

Moreover, the solar $L_0$ path may also assist with the formulation of $H_0$ for $p>n$. Because conserving residual degrees of freedom is so important, tests on the selection (omission) of redundant (informative) variables trigger decisions on which $\beta_i$ to test. \citet[Theorem~2]{zhang09} shows that the earlier a variable enters the $L_0$ path, the more likely it is informative, implying that variables should be tested in rank order. Given the solar path is more robust than lasso path to settings of the irrepresentable condition, sampling noise, multicollinearity, and other issues, it provides more reliable guidance on the order to test the $\beta_i$. $\blacksquare$

\subsection{Complicated dependence structures\label{section:example}}

Another advantage of solar is that the average $L_0$ solution path is more robust to outliers, multicollinearity, and noise in high-dimensional spaces. Thus, solar is likely to be more reliable than other variable selection methods under complicated dependence structures. We illustrate the point with the following two (Bayesian network) examples.

\begin{figure}[ht]
  \centering
  \includegraphics[width=0.35\paperwidth]{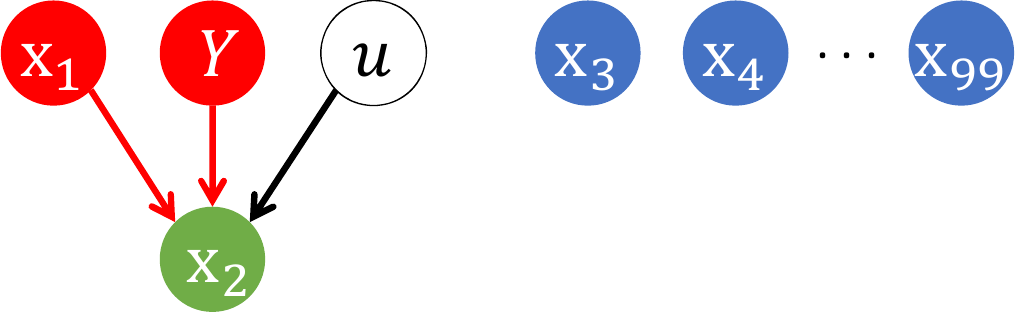}
  \caption{Y is unconditionally uncorrelated with an informative $\mathbf{x}_1$.}
  \label{fig:uncond_example}
\end{figure}

The first example is a common empirical regression problem: \emph{informative variables} that are \emph{unconditionally uncorrelated to} $Y$ in the DGP. In Figure~\ref{fig:uncond_example}, $\mathbf{x}_1$ and $\mathbf{x}_2$ are informative for $Y$, while $\mathbf{x}_1$ and $Y$ are independent. For example, in biostatistics, concussion ($\mathbf{x}_1$) or a brain tumor ($Y$) may cause headaches ($\mathbf{x}_2$), implying that concussion history is when attempting to diagnose a brain tumor. In this setting, Example~2a shows that solar is more reliable than post-lasso rules and variable screening.

\smallskip
\noindent
\textbf{Example 2a.} In Figure~\ref{fig:uncond_example}, there are $100$ variables and $\mathbf{x}_2$ is (causally) generated by its parents $\left\{ \mathbf{x}_1, Y \right\}$ as follows,
\begin{equation}
  \mathbf{x}_2 = \alpha_1 \mathbf{x}_1 + \alpha_2 Y + u,
  \label{eqn:collider_1}
\end{equation}
where $\mathbf{x}_1$ is unconditionally uncorrelated with $Y$, $\mathbf{x}_1$ and $Y$ are both unconditionally and conditionally uncorrelated with the redundant variables $\{\mathbf{x}_3, \ldots, \mathbf{x}_{99}\}$, $\left\{\alpha_1, \alpha_2 \right\}$ are population regression coefficients, and $u$ is a Gaussian noise term. If $Y$ is chosen to be the response variable, the population regression equation is
\begin{equation}
  Y = -\frac{\alpha_1}{\alpha_2} \mathbf{x}_1 + \frac{1}{\alpha_2} \mathbf{x}_2 - \frac{1}{\alpha_2}u.
  \label{eqn:collider_2}
\end{equation}
Note that $\mathbf{x}_1$ and $\mathbf{x}_2$ are both informative variables for $Y$. However, since $\mathbf{x}_1$ is unconditionally uncorrelated with $Y$ in the population, the post-lasso rules [such as the strong rule \citep{tibshirani2012strong} and the safe rule \citep{ghaoui2010safe}] may be prone to rejecting $\mathbf{x}_1$. For a given value of the shrinkage parameter $\lambda$ in grid search, the base strong rule and the safe rule for lasso to reject a selected variable, respectively, satisfies (\ref{eqn:safe_rule}) and (\ref{eqn:strong_rule}):
\begin{eqnarray}
  \left\vert \mathbf{x}_i^T Y \right\vert < & \lambda - \left\Vert \mathbf{x}_i \right\Vert_2 \left\Vert Y \right\Vert_2 \frac{\lambda_{max} - \lambda} {\lambda_{max}} ; \label{eqn:safe_rule} \\
  \left\vert \mathbf{x}_i^T Y \right\vert < & 2\lambda - \lambda_{max} , \label{eqn:strong_rule}
  \label{eqn:post_estmation_rule}
\end{eqnarray}
where the $\mathbf{x}_i$ are standardized and $\lambda_{max}$ is the value of the shrinkage parameter that rejects all the variables. Both rules are based on the unconditional covariance between $\mathbf{x}_i$ and $Y$. For a given value of $\lambda$ (typically selected by CV), lasso will likely select $\mathbf{x}_1$ and $\mathbf{x}_2$ along with redundant variables from $\left\{ \mathbf{x}_3, \ldots, \mathbf{x}_{99} \right\}$ [because the DGP does not violate the IRC]. Since $\mathrm{corr} \left( \mathbf{x}_1, Y \right) = \mathrm{corr} \left( \mathbf{x}_3, Y \right) =  \cdots = \mathrm{corr} \left( \mathbf{x}_{99}, Y \right) = 0$ in the population, the sample value of $\left\vert \mathbf{x}_1^T Y \right\vert$ will be approximately as small as the $\left\vert \mathbf{x}_i^T Y \right\vert$ of any redundant variable. Put another way, $\mathbf{x}_1$ cannot be distinguished from the redundant variables by the value of $\left\vert \mathbf{x}_i^T Y \right\vert$. To ensure $\mathbf{x}_1$ is not rejected by (\ref{eqn:safe_rule}) or (\ref{eqn:strong_rule}), both $\lambda - \left\Vert \mathbf{x}_1 \right\Vert_2 \left\Vert Y \right\Vert_2 \frac{\lambda_{max} - \lambda} {\lambda_{max}}$ and $2\lambda - \lambda_{max}$ must be smaller than $\left\vert \mathbf{x}_1^T Y \right\vert$. However, this will lead to two problems. First, decreasing the right-hand side of (\ref{eqn:safe_rule}) and (\ref{eqn:strong_rule}) will reduce the value of $\lambda$, implying that lasso will select more redundant variables. Second, since $\left\vert \mathbf{x}_1^T Y \right\vert$ will be approximately as small as the $\left\vert \mathbf{x}_i^T Y \right\vert$ of any redundant variable selected by lasso, not rejecting $\mathbf{x}_1$ (by reducing both right-hand side terms) may result in (\ref{eqn:safe_rule}) and (\ref{eqn:strong_rule}) retaining redundant variables.

Variable screening methods \citep{fan2008sure} may also be prone to selecting redundant variables. Screening ranks variables decreasingly based on the absolute values of their unconditional correlations to $Y$, selecting the top $w$ variables (with $w$ selected by CV, bootstrap, or BIC). Since $\mathrm{corr} \left( \mathbf{x}_2, Y \right) \neq 0$ in the population, screening will rank $\mathbf{x}_2$ highly. However, it may not rank $\mathbf{x}_1$ highly because $\mathrm{corr} \left( \mathbf{x}_1, Y \right) = 0$ in the population. Thus, some redundant variables may be ranked between $\mathbf{x}_2$ and $\mathbf{x}_1$, implying that if both $\mathbf{x}_1$ and $\mathbf{x}_2$ are selected, screening will select redundant variables.

The average $L_0$ solution path will not suffer the same problems. For convenience, assume $-\alpha_1 / \alpha_2 > 0$ and $p/n = 100/200$ or smaller. For least angle regression, as $\left\Vert \beta_2 \right\Vert_1$ increases at stage~1 (i.e., as $\mathbf{x}_2$ is `partialled out' of $Y$), the unconditional correlation between $Y - \beta_2 \mathbf{x}_2$ and $\mathbf{x}_1$ will increase above $0$ significantly while the marginal correlation between $Y - \beta_2 \mathbf{x}_2$ and any redundant variable will remain approximately $0$. Thus, in the $L_0$ solution path and, hence, the average $L_0$ solution path, $\mathbf{x}_1$ will be included immediately after $\mathbf{x}_2$ is included. $\blacksquare$

\citet{fan2008sure} and \citet{barut2016conditional} propose two solutions for the problems with variable screening in situations like Example~2a. However,

\begin{itemize}
  \item the first approach \citep[Section~2.2 and~3]{barut2016conditional} assumes the identity of $\mathbf{x}_2$ is known, which is unlikely to be realistic in practical applications. [In Bayesian networks or probabilistic graph modelling, $\mathbf{x}_2$ is known as a \emph{collider}; \citet{barut2016conditional} refer to $\mathbf{x}_2$ as a \emph{hidden signature} variable and denote it by $X_c$];
  \item the second approach \citep[Section~1 and~2.2]{barut2016conditional} suggests randomly trying out several variables to be colliders. The logic is straightforward: randomly trying out a wrong variable (like $\mathbf{x}_2$) to be a collider is harmless because conditioning on that variable will not make $corr(Y,\mathbf{x}_1) \neq 0$, nor will it cause the selection of a redundant variable. Moreover, by repeatedly randomly trying out variables, there is a non-zero probability the correct collider will eventually be uncovered, producing a statistically significant $corr(Y,\mathbf{x}_1) \neq 0$. However, using multiple trials may be inefficient and computationally expensive, especially with high-dimensional data. To improve high-dimensional efficiency, \citet{barut2016conditional} suggests trying out several variables simultaneously. However, if $corr(Y, \mathbf{x}_1) \neq 0$ were discovered after trying out, say, $\left\{\mathbf{x}_2,\mbox{other variables}\right\}$, it would still be necessary to decide which of $\left\{\mathbf{x}_2,\mbox{other variables}\right\}$ are redundant, meaning variable selection is not completed.
\end{itemize}

\medskip

The second example illustrates another common problem in empirical regression: \emph{redundant variables} that are \emph{unconditionally correlated to} $Y$ in the DGP. In Figure~\ref{fig:cond_example}, the problem occurs because $\mathbf{x}_3$ and $Y$ are determined by common variables. For example, house rent ($Y$) and food expenditure ($\mathbf{x}_3$) are both determined by income ($\mathbf{x}_1$) and saving ($\mathbf{x}_2$), yet $\mathbf{x}_3$ is redundant if $\mathbf{x}_1$ and $\mathbf{x}_2$ are used to predict $Y$. In this setting, Example~2b illustrates that the strong rule, base rule, and variable screening methods may struggle to reject the redundant $\mathbf{x}_3$ even when IRC is satisfied. By contrast, solar will be less prone to selecting redundant variables.

\begin{figure}[ht]
    \centering
    \includegraphics[width=0.35\paperwidth]{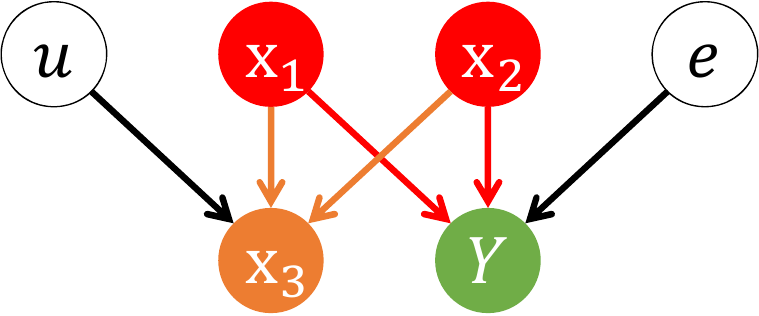}
    \caption{$Y$ is unconditionally correlated with a redundant $\mathbf{x}_3$.}
    \label{fig:cond_example}
  \end{figure}

\smallskip
\noindent
\textbf{Example 2b.} Figure~\ref{fig:cond_example} depicts the following confounding structure,
\begin{equation}
	\begin{cases}
    \mathbf{x}_3 = \frac{1}{3} \mathbf{x}_1 + \frac{1}{3} \mathbf{x}_2 + \frac{\sqrt{7}}{3} u, \\
    Y = \frac{7}{10} \mathbf{x}_1 +  \frac{2}{10} \mathbf{x}_2 +  \frac{\sqrt{47}}{10} e, \\
	\end{cases}
	\label{eqn:example_4}
\end{equation}
where $\mathbf{x}_1$ and $\mathbf{x}_2$ cause both $Y$ and $\mathbf{x}_3$, implying that $\mathbf{x}_3$ is unconditionally correlated to $Y$; $\mathbf{x}_1$, $\mathbf{x}_2$, $u$ and $e$ are independent; $\mathbf{x}_3$ is independent from $e$; $Y$ is independent from $u$; and all variables are standardized.

For large $n$, when the sample correlations are close to their population values, the sample marginal correlations to $Y$ are:
\begin{equation}
  \begin{aligned}
    \mathrm{corr} \left( \mathbf{x}_1, Y \right)  = & \;0.7, \\
    \mathrm{corr} \left( \mathbf{x}_3, Y \right)  = & \;\mathrm{corr} \left( \frac{1}{3} \mathbf{x}_1 + \frac{1}{3} \mathbf{x}_2, \frac{7}{10} \mathbf{x}_1 +  \frac{2}{10} \mathbf{x}_2 \right)
    = 0.3, \\
    \mathrm{corr} \left( \mathbf{x}_2, Y \right)  = & \;0.2. \\
  \end{aligned}
\end{equation}
Because $\mathbf{x}_2$ ranks below $\mathbf{x}_1$ and $\mathbf{x}_3$ in terms of marginal correlations to $Y$, the variable screening method must select all $3$ variables---including the redundant $\mathbf{x}_3$---to avoid omitting $\mathbf{x}_2$. The base strong rule and safe rule may also have difficulty rejecting $\mathbf{x}_3$. Since $\mathrm{corr} \left( \mathbf{x}_3, Y \right)>\mathrm{corr} \left( \mathbf{x}_2, Y \right)$, if lasso selects $\mathbf{x}_3$ and $\mathbf{x}_2$ and the strong (or safe) rule is used to reject $\mathbf{x}_3$, $\mathbf{x}_2$ will also be rejected.

Forward regression, solar, and lasso will not make the same error. Because (\ref{eqn:example_4}) does not violate the IRC, variable-selection consistency of forward regression, lars, and lasso is assured from the theoretical results of \citet{zhang09} and \citet{zhaoyu06}. In forward regression, $\mathbf{x}_1$ will be included  at the first stage. After controlling for $\mathbf{x}_1$, the partial correlations (for large $n$) of both $\mathbf{x}_2$ and $\mathbf{x}_3$ with $Y$ are:
\begin{equation}
  \begin{aligned}
    \mathrm{corr} \left( \mathbf{x}_2, Y \vert \mathbf{x}_1 \right)  = & \;\mathrm{corr} \left( \mathbf{x}_2, \frac{2}{10} \mathbf{x}_2 \right)
    = 0.2, \\
    \mathrm{corr} \left( \mathbf{x}_3, Y \vert \mathbf{x}_1 \right)  = & \;\mathrm{corr} \left( \frac{1}{3} \mathbf{x}_1 + \frac{1}{3} \mathbf{x}_2, \frac{2}{10} \mathbf{x}_2 \right)
    = 0.0667. \\
  \end{aligned}
\end{equation}
Because $\mathrm{corr}(\mathbf{x}_2, Y \vert \mathbf{x}_1)>\mathrm{corr}(\mathbf{x}_3, Y \vert \mathbf{x}_1)$, forward regression will include $\mathbf{x}_2$ not $\mathbf{x}_3$ at the second stage. After controlling for both $\mathbf{x}_1$ and $\mathbf{x}_2$, the remaining variation in $Y$ is due to $e$, which $\mathbf{x}_3$ cannot explain. Thus, CV or BIC will terminate forward regression after the second stage and $\mathbf{x}_3$ will not be selected. Similarly, because solar relies on the average $L_0$ path, it will include $\mathbf{x}_1$ and $\mathbf{x}_2$ but not $\mathbf{x}_3$. $\blacksquare$

\bigskip
Essentially, the strong rule, safe rule, and variable screening struggle in Examples~2a and~2b because they rely on unconditional correlations to $Y$, whereas informative variables in regression analysis are defined in terms of conditional correlations. In many scenarios, unconditional and conditional correlations are aligned. However, when they are not, variable selection based conditional correlation is better placed to select the informative variables.

\citet{fan2008sure} propose redeeming variable screening on $Y$ by first selecting variables with high unconditional correlations to $Y$ and then running a lasso of the residuals on the dropped variables. By contrast, solar completes variable selection in a single pass of conditional correlation ranking, reducing computational costs. Moreover, the \citet{fan2008sure} approach does not solve Example~2b type problems. At the first step, variables with high unconditional correlations to $Y$ will be selected, including the redundant $\mathbf{x}_3$. Selecting redundant variables will be more serious when $Y$ has multiple $\mathbf{x}_3$-like siblings and in complicated dependence structures where multicollinearity results in inaccurate estimates of the coefficients and standard errors in finite samples. In short, solar is likely to be more computationally efficient and better at variable selection in settings with complicated dependence structures.


\subsection{Robustness to the IRC \label{subsection:irc}}

Solar is more robust to different settings of the IRC than the lasso. The IRC is considered to be sufficient and almost necessary for accurate lasso variable selection \citep{zhang09}. Here, we ignore lasso rules and variable screening since, as discussed above, their selection accuracy may be compromised by a reliance on unconditional correlations to $Y$. We define the IRC as in \citet{zhang09}.

\begin{definition}[IRC]
  Given $F \subset \left\{ 1, \ldots, p \right\}$, define $X_F$ to be the $n \times \left\vert F \right \vert$ matrix with only the full set of informative variables. Define
    \begin{align}
    \mu \left( F \right) = & \max \left\{ \left\Vert \left( \left( X_F \right)^T X_F \right)^{-1} \left( X_F \right)^T \mathbf{x}_j \right\Vert_1 \; \vert \; \forall j \not\in F \right\}. \notag
    \end{align}
  Given a constant $1 \geqslant \eta > 0$, the \emph{strong} irrepresentable condition is satisfied if $\mu \left( F \right) \leqslant 1 - \eta$ and the \emph{weak} irrepresentable condition is satisfied if $\mu \left( F \right) < 1$.$\blacksquare$
\end{definition}

\smallskip
\noindent
\textbf{Example 3.} Modify the DGP in Example~2b to match the \citet{zhaoyu06} simulations. Thus, $n = 200$, $p = 50$, and $\{\mathbf{x}_0, \ldots, \mathbf{x}_4, \mathbf{x}_6, \ldots, \mathbf{x}_{50}\}$ are generated from a zero-mean, unit-variance multivariate Gaussian distribution, where all the correlation coefficients are $0.5$. The DGP of $Y$ and $\mathbf{x}_5$ is
\begin{equation}
	\begin{cases}
    \mathbf{x}_5 = \omega \mathbf{x}_0 + \omega \mathbf{x}_1 + \gamma\cdot \sqrt{1 - 2\omega^2} \\
    Y = 2 \mathbf{x}_0 + 3\mathbf{x}_1 + 4 \mathbf{x}_2 + 5 \mathbf{x}_3 + 6 \mathbf{x}_4 + e \\
	\end{cases}
	\label{eqn:dgp_x5}
\end{equation}
where $\omega \in \mathbb{R}$, while $\gamma$ and $e$ are both standard Gaussian noise terms, independent from each other and all the other variables. Compared with Example~2b, this DGP increases the challenge of accurate selection by increasing the number of redundant variables from 1 to 46, $\{\mathbf{x}_5, \ldots, \mathbf{x}_{50}\}$. This DGP also makes it straightforward to control the IRC through $\omega$, which affects the value of $\mu \left( F \right)$.

\begin{figure}
  \centering
  \subfloat[\label{fig:solar_ic_type-II1}$\omega = 1/4,\;\mu\left(F\right)=1/2$, lasso]
  {\includegraphics[width=0.25\paperwidth]{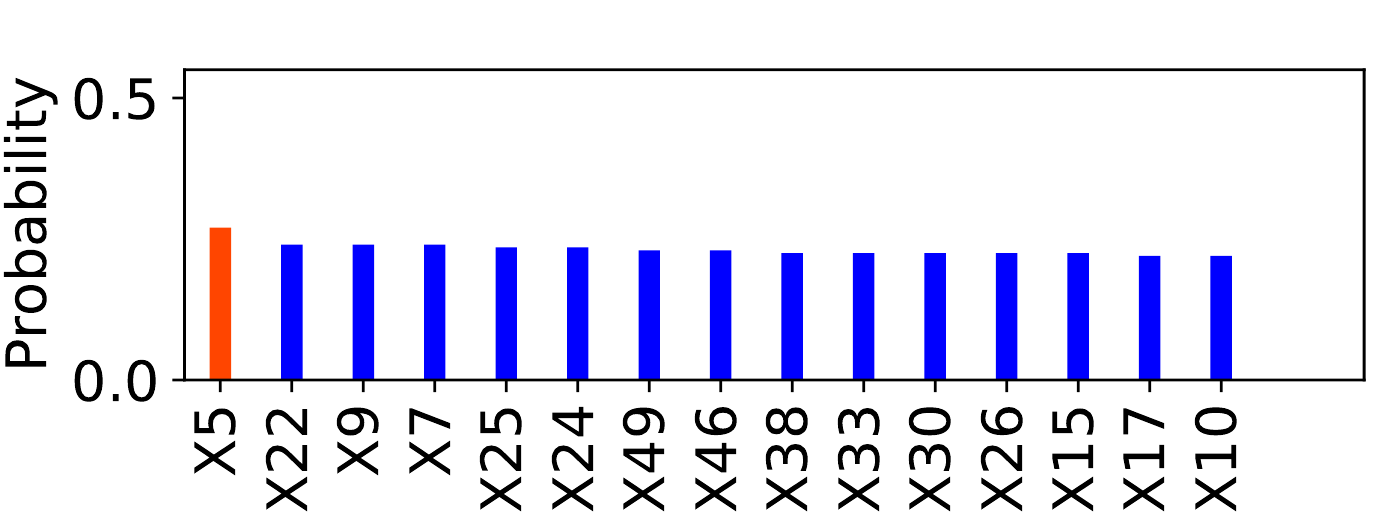}}
  \subfloat[\label{fig:solar_ic_type-II2}$\omega = 1/3,\;\mu\left(F\right)=2/3$, lasso]
  {\includegraphics[width=0.25\paperwidth]{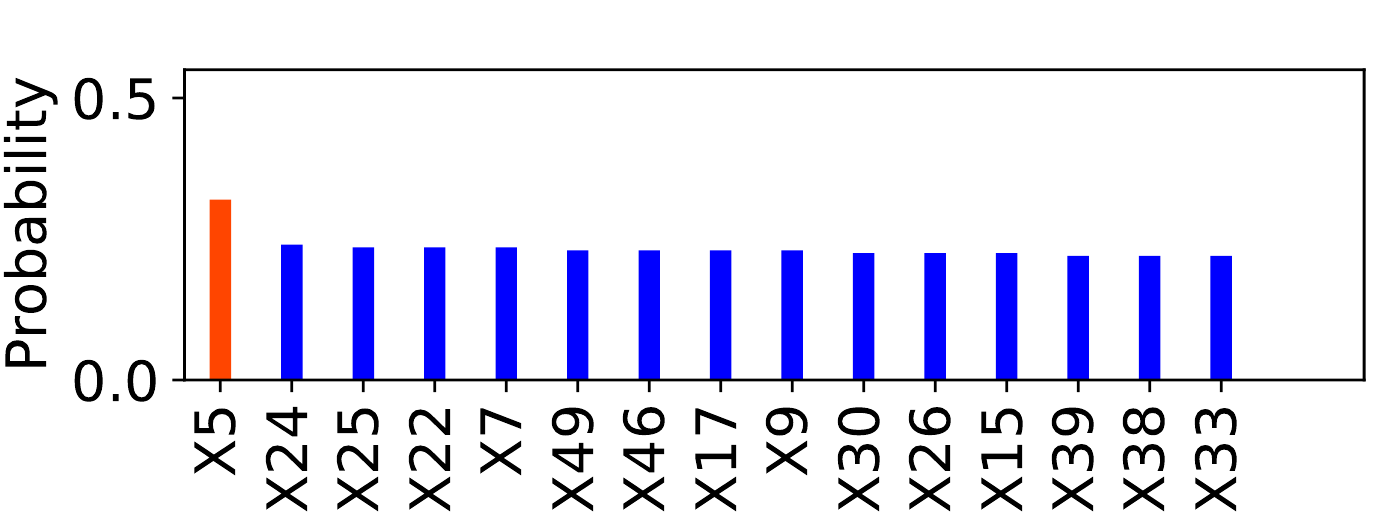}}
  \subfloat[\label{fig:solar_ic_type-II3}$\omega = 1/2,\;\mu\left(F\right)=1$, lasso]
  {\includegraphics[width=0.25\paperwidth]{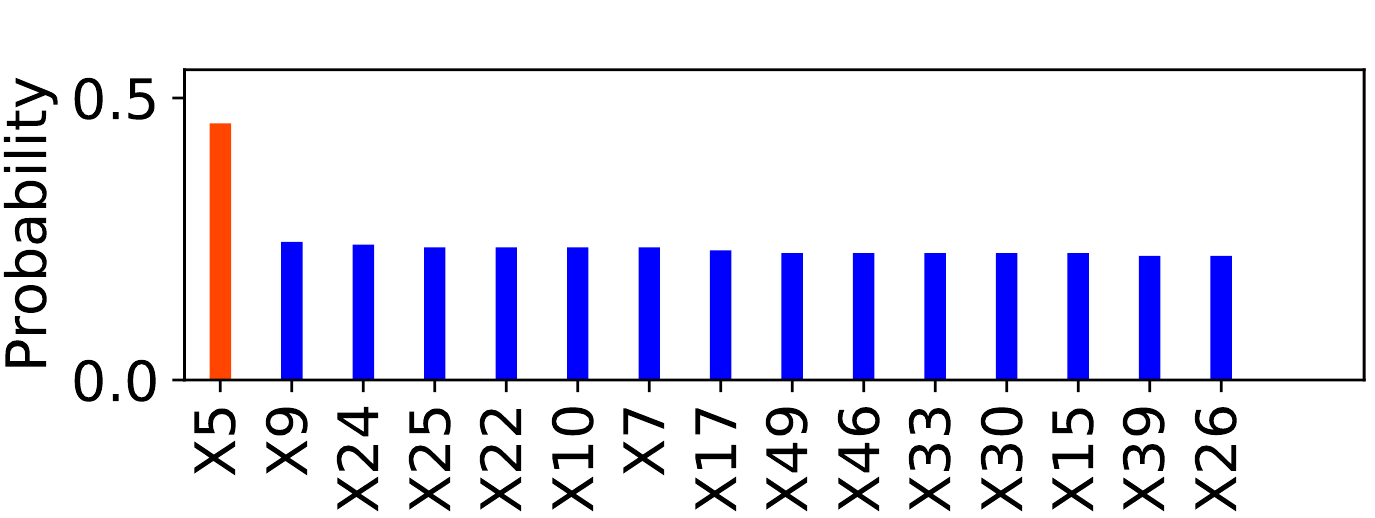}}

  \subfloat[\label{fig:solar_ic_type-II7}$\omega = 1/4,\;\mu\left(F\right)=1/2$, solar]
  {\includegraphics[width=0.25\paperwidth]{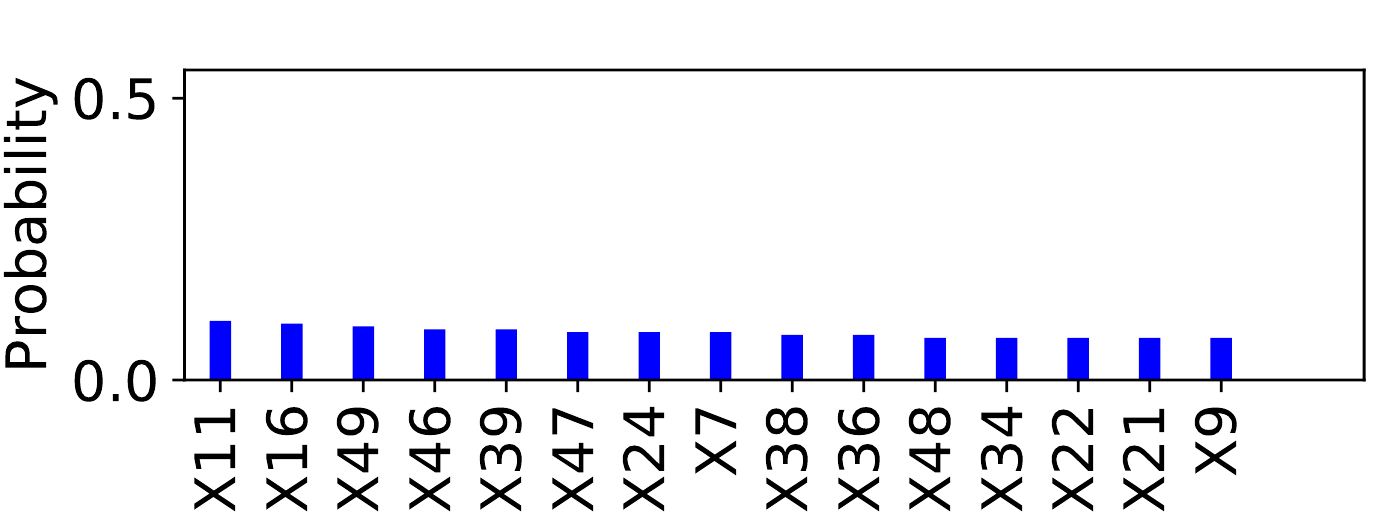}}
  \subfloat[\label{fig:solar_ic_type-II8}$\omega = 1/3,\;\mu\left(F\right)=2/3$, solar]
  {\includegraphics[width=0.25\paperwidth]{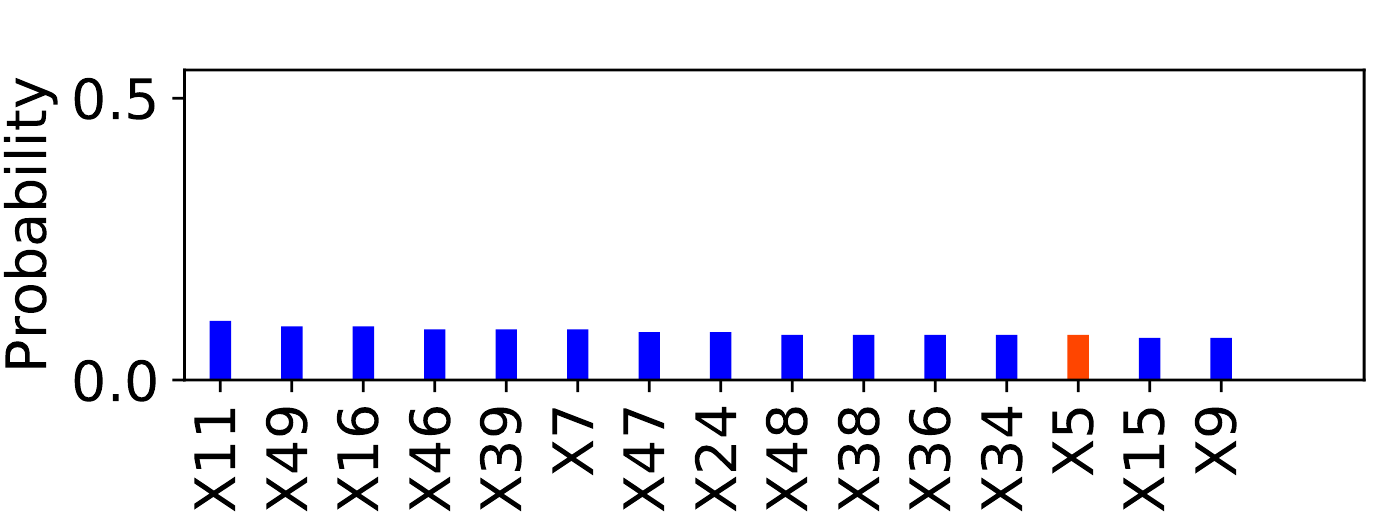}}
  \subfloat[\label{fig:solar_ic_type-II9}$\omega = 1/2,\;\mu\left(F\right)=1$, solar]
  {\includegraphics[width=0.25\paperwidth]{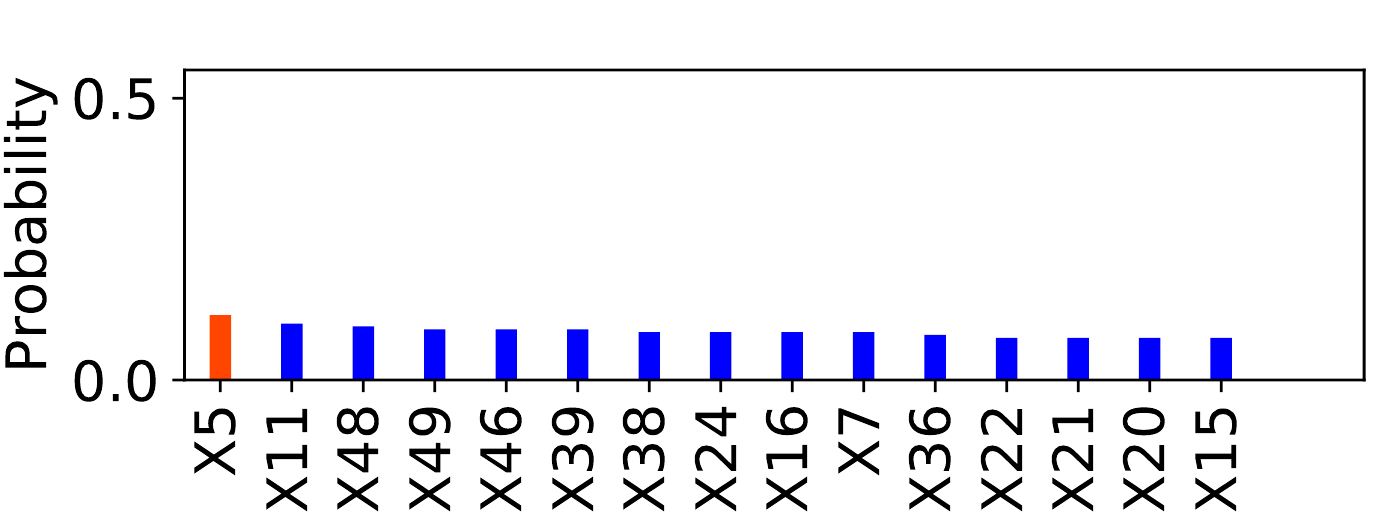}}
  \caption{Probability of including redundant variables (top 15) in simulation~2 ($\mathbf{x}_5$ in orange).}
  \label{fig:solar_ic_type-II}
\end{figure}

In (\ref{eqn:dgp_x5}), the IRC only affects the redundant $\mathbf{x}_5$. Hence, we focus on the probability of incorrectly selecting $\mathbf{x}_5$ in 200 repetitions. By setting $\omega$ to either $1/4$, $1/3$, or $1/2$, the population value of $\mu \left( F \right)$ changes, respectively, to $1/2$, $2/3$, or $1$, gradually increasing the difficulty of rejecting the redundant $\mathbf{x}_5$.

Figure~\ref{fig:solar_ic_type-II} displays the simulation results. When $\mu \left( F \right) = 1/2$, lasso wrongly includes $\mathbf{x}_5$ with probability $0.25$. By contrast, $\mathbf{x}_5$ is not among the top 15 variables selected by solar, implying a probability less than $0.1$. When $\mu \left( F \right)$ increases to $2/3$, the probability lasso includes $\mathbf{x}_5$ increases to around $0.3$. When $\mu \left( F \right)$ increases to $1$ in the population and strong IRC is violated, the probability lasso includes $\mathbf{x}_5$ rises to almost $0.5$. By contrast, the probability solar includes $\mathbf{x}_5$ is below $0.1$ even when $\mu\left(F\right)=1$. The results illustrate that solar is more robust to different settings of the IRC. $\blacksquare$


\section{Solar advantages over subsample variable selection\label{section:comp}}

In this section, we shift our focus to simulation. We demonstrate that: (i) solar offers significant improvements over lasso-type algorithms in terms of variable selection sparsity and accuracy; (ii)~replacing lasso with solar in bootstrap selection drastically reduces the computation load, measured by runtime. We choose the simulation settings so that, as far as possible, the comparisons are fair, representative, and generalizable. Our overall goal is to enable \emph{ceteris paribus} comparisons between solar and state-of-the-art lasso algorithms.

\subsection{Simulation competitors}

We consider a subset of lasso-type algorithms for comparison to solar. Firstly, some lasso modifications (e.g., fused lasso, grouped lasso) are designed to solve specific empirical problems that are not relevant to our paper. Secondly, it may be difficult to determine how much some variants outperform lasso.\footnote{For example, while \citet{jia2010model} show numerically that elastic net has slightly better variable-selection accuracy than lasso, they also find that ``when the lasso does not select the true model, it is more likely that the elastic net does not select the true model either'' (a point we verify in Section~\ref{section:application}). While simulations in \citet{zou2006adaptive} show that adaptive lasso outperforms lasso when $p/n<1$, it requires first computing the OLS estimates of all $\mathbf{x}_i$ coefficients, which is difficult when $p/n>1$.} Since both solar and lasso may be evaluated via least angle regression and coordinate descent, many other lasso modifications can be directly applied to solar, as discussed in Section~\ref{subsec:variant}. We do not consider information criteria for shrinkage parameter tuning. \citet{scikit-learn} points out that information criteria are over-optimistic and require a proper estimation of the degrees of freedom for the solution. Moreover, information criteria are derived asymptotically and tend to break down when the problem is badly conditioned (e.g., $p > n$).\footnote{See \url{https://scikit-learn.org/stable/modules/linear_model\#lasso.html} for details.}

Solar competes with $10$-fold, cross-validated lasso (denoted `lasso' for short), following the \citet{friedman2001elements} simulations that show 10 folds balances the bias-variance trade-off in CV error minimization. We set the number of generated subsamples ($K$) in Algorithm~\ref{algo:APE-lar} to $3$ since $K>3$ has only negligible effects. Because least-angle regression and coordinate descent yield similar selection results for solar and lasso, we combine the lars and coordinate descent results for solar and report only the runtime for lars lasso (see Supplementary Material B).

We also include bootstrap selection algorithms in the comparisons. A bootstrap selection repeats lasso multiple times across bootstrap subsamples to produce a set of averaged (or accumulated) selection results. Given the similarities among lasso bootstrap selection methods, we choose the \citet{bach2008bolasso} bootstrap lasso (\emph{bolasso}) to be the competitor to solar. \citet{bach2008bolasso} proposes two bolasso algorithms: bolasso-H and bolasso-S; both are competitors in the simulations. Bolasso-H selects only variables that are selected in all bootstrap subsamples, i.e., the subsample selection frequency threshold, $f=1$. Bolasso-S selects variables that are selected in 90\% of the bootstrap subsamples ($f=0.9$). \citet{bach2008bolasso} finds that bolasso selection and prediction performance improves with the number of subsamples. To ensure a rigorous challenge for solar, we set the number of bootstrap subsamples in bolasso to $256$, the maximum in the \citet{bach2008bolasso} simulations. Moreover, \citet{meinshausen2010stability} points out that bolasso relies on choosing the $\lambda$ value on bootstrap subsamples. If the $\lambda$ value is unecessarily large on more than $10\%$ of all bootstrap subsamples, bolasso-H and bolasso-S will omit informative variables. Given the fact that the optimal value of $\lambda$ may change substantially in high dimenisons, we use $10$-fold cross validation to tune $\lambda$ in each bootstrap subsample.

We also consider a bootstrap solar selection \emph{(bsolar)}, which executes solar on each bootstrap subsample and computes the selection frequency for each variable across all bootstrap subsamples. To ensure that any performance difference is due to replacing lasso with solar in the bootstrap selection system, bolasso and bsolar use the same subsample selection frequency threshold. Thus, we evaluate 2~versions of bsolar: bsolar-H ($f=1$) and bsolar-S ($f=0.9$). We use the notation bsolar-$m$H and bsolar-$m$S, where $m$ is the number of subsamples used to compute the selection frequency.

\subsection{Simulation settings}

The DGP for the simulations is as follows. The $p$ covariates in $X \in \mathbb{R}^{n \times p}$ are generated from a zero-mean, multivariate Gaussian distribution, with all off-diagonal elements in the covariance matrix equal to~0.5. The first 5 variables in $X$ are informative; the remaining $p-5$ variables are redundant. The response variable $Y \in \mathbb{R}^{n \times 1}$ is:
\begin{equation}
  Y =  2 \mathbf{x}_0 + 3 \mathbf{x}_1 + 4 \mathbf{x}_2 + 5 \mathbf{x}_3 + 6 \mathbf{x}_4  + e,
  \label{eqn:pop_model}
\end{equation}
where $e\in \mathbb{R}^{n \times 1}$ is a standard Gaussian noise term. All data points are independently and identically distributed. Each $\mathbf{x}_i$, $i=1,\ldots,p$, is independent from the noise term $e$, which is standard Gaussian. Simulations are repeated 200 times with fixed Python random generators across simulations.

We vary the data dimensions $p/n$ as follows. In the first block of simulations, $p/n$ approaches $0$ from above, corresponding to the classical $p<n$ setting. In the second block, $p/n$ approaches $1$ from above, corresponding to high dimension settings. In the third block, $p/n=2$ as $\log(p)/n$ slowly approaches $0$, corresponding to ultrahigh dimension settings, i.e., where $(p-n)\rightarrow\infty$.

We compare the performance of solar and lasso in terms of sparsity and accuracy of variable selection and on the runtime. Sparsity is measured by the mean number of selected variables. Discovery accuracy is measured by the mean number of \emph{informative} selected variables. Purge accuracy is measured by the mean number of \emph{redundant} selected variables (equal to sparsity minus discovery accuracy). Runtime is measured by mean CPU time.  The raw simulation results are available in the supplementary file.

\subsection{Programming languages, parallelization, and hardware}

To ensure a credible comparison between solar and the lasso competitors, we choose the hardware and software settings to maximize the computation speed of lasso. We show that, even under the ideal computation environment for lasso, solar exhibits a substantial runtime advantage.

To maximize computation speed, we use \texttt{Numpy}, \texttt{Scipy}, and \texttt{Cython}---all well-known for performance and speed---to outsource all numerical and matrix operations to the Intel Math Kernel Library, currently the fastest and most accurate C++/Fortran library for CPU numerical operations.

To reduce the possibility of CPU and RAM bottlenecks in parallel computing of lasso and bootstrap lasso, we code in Python rather than R. \citet{donoho201750} claims: ``R has the well-known reputation of being less scalable than Python to large problem sizes''. Given the simulations repeat solar, lasso, and bootstrap lasso many times to arrive at representative performance measures, choosing Python over R mitigates the impact of hardware limitations. Computations are executed with an Intel Xeon W-3245 CPU with 3.2GHz base frequency and 64GB RAM, further reducing the possibility of CPU-RAM bottlenecks.

To guarantee the programming quality of the lasso implementation, we source lasso and bootstrap lasso from the Sci-kit learn library \citep{scikit-learn} of efficient machine-learning tools.\footnote{Detail is available at \url{https://scikit-learn.org/stable/}.} Used widely in research and industry, Sci-kit learn also uses \texttt{Numpy}, \texttt{Scipy}, and \texttt{Cython} to delegate all numerical and matrix operations to Fortran/C++.

Lastly, to optimize computation and avoid large overheads, we implement multi-core parallelization. Each realization of lasso requires 10 repetitions of lars or coordinate descent to compute the CV error of each $\lambda$ value. Thus, we design a parallel architecture to assign one repetition per CPU core, maximizing the computation speed for lasso.

\subsection{Comparison of sparsity, accuracy, and time complexity \label{subsection:suml1}}

Table~\ref{table:sim_1} summarizes average selection performance.\footnote{Detailed histograms are available in the supplementary file.} While all competitors always include the 5 informative variables, solar outperforms lasso in terms of sparsity in every $p/n$ scenario, implying superior ability to limit the selection of redundant variables. Notably, as $p/n\rightarrow1$, lasso sparsity deteriorates while solar sparsity improves, further confirming the advantage of path averaging. While the sparsity of all competitors deteriorates as $\log(p)/n\rightarrow0$, solar maintains a clear advantage over lasso.

\begin{table}[ht]
\centering
\caption{Simulation results for sparsity and accuracy.\label{table:sim_1}}
\resizebox{0.98\textwidth}{!}{%
\renewcommand{\arraystretch}{0.7}
\begin{tabular}{l ... ... ...}
  \toprule
  & \multicolumn{3}{c}{$p/n\rightarrow0$}
  & \multicolumn{3}{c}{$p/n\rightarrow1$}
  & \multicolumn{3}{c}{$\log(p)/n\rightarrow0$} \\
  \cmidrule(lr){2-4} \cmidrule(lr){5-7} \cmidrule(lr){8-10}
  & \multicolumn{1}{c}{$\frac{100}{100}$} & \multicolumn{1}{c}{$\frac{100}{150}$} & \multicolumn{1}{c}{$\frac{100}{200}$}
  & \multicolumn{1}{c}{$\frac{150}{100}$} & \multicolumn{1}{c}{$\frac{200}{150}$} & \multicolumn{1}{c}{$\frac{250}{200}$}
  & \multicolumn{1}{c}{$\frac{400}{200}$} & \multicolumn{1}{c}{$\frac{800}{400}$} & \multicolumn{1}{c}{$\frac{1200}{600}$}  \\
  \midrule
  \multicolumn{9}{l}{\emph{mean number of selected variables}}\\
  \hspace*{5mm}lasso        & 20.4 & 19.5 & 19.7 & 23.1 & 24.1 & 27.2 & 30.7 & 36.7 & 37.4 \\
  \hspace*{5mm}solar        & 10.5 &  9.3 &  9.1 & 10.7 &  9.8 &  8.7 & 11.4 & 16.1 & 18.5 \\
  \\ [-8pt]
  \hspace*{5mm}bolasso-S    &  5.5 &  6.4 &  6.5 &  5.5 &  6.4 &  6.5 &  5.7 &  6.6 &  7.6 \\
  \hspace*{5mm}bolasso-H    &  5   &  5   &  5   &  5   &  5   &  5   &  5   &  5   &  5   \\
  \\ [-8pt]
  \hspace*{5mm}bsolar-3S/3H &  5.4 &  5.2 &  5.1 &  5.4 &  5.2 &  5.1 &  5.3 &  5.8 &  6   \\
  \hspace*{5mm}bsolar-5S/5H &  5.2 &  5.1 &  5   &  5.2 &  5.1 &  5   &  5.1 &  5.2 &  5.4 \\
  \hspace*{5mm}bsolar-10S   &  5.2 &  5.1 &  5   &  5.2 &  5.1 &  5   &  5.1 &  5.2 &  5.3 \\
  \hspace*{5mm}bsolar-10H   &  5   &  5   &  5   &  5   &  5   &  5   &  5   &  5   &  5.1 \\
  \\ [-8pt] \multicolumn{9}{l}{\emph{mean number of selected informative variables}}\\
  \hspace*{5mm}lasso        &  5   &  5   &  5   &  5   &  5   &  5   &  5   &  5   &  5   \\
  \hspace*{5mm}solar        &  5   &  5   &  5   &  5   &  5   &  5   &  5   &  5   &  5   \\
  \hspace*{5mm}bolasso-S/H  &  5   &  5   &  5   &  5   &  5   &  5   &  5   &  5   &  5   \\
  \hspace*{5mm}bsolar-3S/3H/5S/5H/10S/10H & 5 & 5 & 5 & 5 & 5 & 5 & 5 & 5 & 5 \\
  \bottomrule
  \end{tabular}}
\end{table}

Table~\ref{table:sim_1} also reveals several advantages of solar over lasso in bootstrap selections.
\begin{itemize}
  \item In terms of variable selection, bolasso-S stands out with the poorest sparsity while the others perform almost identically.
  \item Solar and bsolar exhibits a considerable computational advantage. We show in Section~\ref{subsection:comp} that solar imposes less than $1/3$ of the lasso computation load, implying that bsolar-3 has the same computation load as lasso. Given bolasso requires 256 subsample lasso repetitions while bsolar-3 has the same computation load as one lasso realization, bsolar reduces subsample repetitions by 99\% relative to bolasso (assuming a time complexity measure like $O(n^2)$ and $p>n$ for lasso).
  \item Similar findings apply to the comparison between bsolar and lasso stability selection. Bsolar, bolasso-H, and stability selection ($f>0.9$) return very similar sparsity and accuracy (on average selecting all informative variables and very rarely including an redundant variable). However, lasso stability selection implements $100$ subsample repetitions respectively while bsolar-3 only requires $3$. Even though the size of the bootstrap subsample in stability selection is $n/2$ (substantially smaller than the bootstrap sample size of bsolar, which is $n$), the time complexity analysis of \citet{meinshausen2010stability} still implies that bsolar-3 produces a reduction of at least $67$-$82\%$ in computation load relative to lasso stability selection. Such amount of computation time reduction is crucial in large scale applications like DNA sequencing, natural language processing, imagine processing, and MRI neuroimaging [where each observation (image) often contains more than $10^6$ pixels as candidate variables, and the total data size can easily go beyond 1GB even with limited $n$.]. The amount of computation reduction can be even more substantial if the application requires a certain lasso/solar variation like "group", "fused", or "elastic net" (discussed in Section~\ref{subsection:variant}).   
\end{itemize}

\subsection{Explanation of the efficiency discrepancy between bolasso-bsolar}

The efficiency of bsolar is due to its unique multi-layer variable ranking scheme. While bsolar and bolasso both generate bootstrap subsamples, bsolar uses a different bootstrap variable selection procedure. Specifically,

\begin{itemize}

  \item     solar executes Algorithm~\ref{algo:APE-lar} (or \ref{algo:APE-cd}) on each bootstrap subsample and ranks variables using the average $L_0$ path, which we call the \emph{internal ranking}. The internal ranking identifies the strongest signals on each bootstrap subsample.

  \item     bsolar collates the internal ranking results to produce an overall ranking, which we call the \emph{external ranking}. The external ranking identifies the strongest signals on the majority of bootstrap subsamples.
\end{itemize}

The multi-layer method has several advantages over the usual one-layer ranking methods, such as bootstrap lasso and lasso stability selection \citep{fan2008sure, hall2009usingb, hall2009using, li2012robust, li2012feature}.
\begin{itemize}
  \item First, one-layer methods rank variables on the whole sample. By contrast, the internal ranking uses the average $L_0$ path, which, as discussed in Section~2.1, improves robustness to multicollinearity, noise, and sample size.
  \item Second, as shown in Section~\ref{section:example}, internal ranking avoids issues caused by complicated dependence structures that other (unconditional) ranking methods cannot.
  \item Most important, multi-layer ranking reduces the number of bootstrap repetitions without compromising accuracy. One-layer methods select variables immediately after ranking. Our method performs a second external ranking that, by detecting persistent signals, is more tolerant of subsample variation: if $\mathbf{x}_i$ is wrongly selected or omitted in the internal ranking, there is still a large probability that the mistake will be corrected in the external ranking. While stability selection and bolasso require, respectively, 100 and 256 repetitions to average out lasso selection issues, bsolar requires only 3-10 bootstrap repetitions to confirm the solar variable ranking.
\end{itemize}

\begin{table}[!htb]
  \caption{Subsample variable selection frequencies for bolasso and bsolar-10.}
  \label{table:subsample_select_freq}
  \begin{minipage}[t]{.55\linewidth}
    \small
    \subfloat[bolasso]{%
    \label{table:subsample_select_freq_1}
			\renewcommand{\arraystretch}{0.7}
			\begin{tabular}{cl}
				\toprule
				frequency & variables \\
				\midrule
				$\geqslant 1.00$ & $\mathbf{x}_4, \mathbf{x}_3, \mathbf{x}_2, \mathbf{x}_1, \mathbf{x}_0$ \\
				$\geqslant 0.88$ & $\mathbf{x}_4, \mathbf{x}_3, \mathbf{x}_2, \mathbf{x}_1, \mathbf{x}_0, \mathbf{x}_{28}$ \\
				$\geqslant 0.84$ & $\mathbf{x}_4, \mathbf{x}_3, \mathbf{x}_2, \mathbf{x}_1, \mathbf{x}_0, \mathbf{x}_{28}, \mathbf{x}_{71}$\\
				$\geqslant 0.76$ & $\mathbf{x}_4, \mathbf{x}_3, \mathbf{x}_2, \mathbf{x}_1, \mathbf{x}_0, \mathbf{x}_{28}, \mathbf{x}_{71}, \mathbf{x}_{91}$\\
				$\geqslant 0.70$ & $\mathbf{x}_4, \mathbf{x}_3, \mathbf{x}_2, \mathbf{x}_1, \mathbf{x}_0, \mathbf{x}_{28}, \mathbf{x}_{71}, \mathbf{x}_{91}, \mathbf{x}_{94}$\\
				$\geqslant 0.69$ & $\mathbf{x}_4, \mathbf{x}_3, \mathbf{x}_2, \mathbf{x}_1, \mathbf{x}_0, \mathbf{x}_{28}, \mathbf{x}_{71}, \mathbf{x}_{91}, \mathbf{x}_{94}, \mathbf{x}_{70}, \mathbf{x}_{40}$ \\
        $\vdots$ & $\vdots$ \\
				\bottomrule
		\end{tabular}}
  \end{minipage}
  \begin{minipage}[t]{.5\linewidth}
    \small
    \subfloat[bsolar-10]{%
    \label{table:subsample_select_freq_2}
			\renewcommand{\arraystretch}{0.7}
			\begin{tabular}{cl}
				\toprule
				frequency & variables \\
				\midrule
				$\geqslant 1.00$ & $\mathbf{x}_4, \mathbf{x}_3, \mathbf{x}_2, \mathbf{x}_1, \mathbf{x}_0$ \\
				$\geqslant 0.10$ & $\mathbf{x}_4, \mathbf{x}_3, \mathbf{x}_2, \mathbf{x}_1, \mathbf{x}_0, \mathbf{x}_{91}, \mathbf{x}_{71}$ \\
                $= 0$ & all other variables \\
				\bottomrule
		\end{tabular}
    }
  \end{minipage}
\end{table}

Furthermore, as shown in Table~\ref{table:subsample_select_freq}, bsolar produces a shorter and more accurate list of subsample variable selection frequencies. Table~\ref{table:subsample_select_freq_1} breaks down the subsample selection frequency list from 256 subsamples for one bolasso realization with $p/n=100/200$. Due to the length of the list, we report only subsample selection frequencies $\ge0.69$. With only one layer of ranking, bolasso is unable to separate informative from redundant variables even with 256 subsample repetitions. The frequency discrepancy for bolasso between the highest-ranking redundant ($\mathbf{x}_{28}$) and the lowest-ranking informative variable ($\mathbf{x}_0$) is only $0.12$. By contrast, Table~\ref{table:subsample_select_freq_2} shows bsolar-10 returns a much shorter list with a frequency discrepancy between the highest-ranking redundant ($\mathbf{x}_{91}$) and the lowest-ranking informative variable ($\mathbf{x}_0$) of $0.9$. To increase the discrepancy between the lowest ranked informative and highest ranked redundant variables for bolasso, \citet{bach2008bolasso} suggests raising the number of subsample repetitions. However, increasing repetitions will raise the bolasso computation load in high-dimensional spaces, increasing the advantage of bsolar.

\subsection{Computation time comparison \label{subsection:comp}}

The time complexity of an algorithm indicates only how computation time changes as data size (parameterized by $n$, $p$, and $K$) increases. Time complexity analysis omits many other computation parameters (such as hardware specification), suggesting it may substantially underestimate computation time difference of two algorithms in real-world problems. Hence, in this section we compare computation efficiency in terms of CPU times.

Since the computation load for lars or coordinate descent on a given sample is fixed, we may use the number of lars or coordinate descents to approximate the computation load for solar and lasso. For comparison, we compute solar with $K$ subsamples and lasso with $K$-fold cross-validation. As shown in Algorithm~\ref{algo:APE-lar} and \ref{algo:APE-cd}, solar computes one lars or coordinate descent on each subsample $(X^k, Y^k)$, which implies $K=3$ lars or coordinate descents to compute $\widehat{q}$ and one more pass to compute $c^*$ for variable selection. Lasso requires computing $K=10$ lars or coordinate descents to optimize the tuning parameter and, given the optimal tuning parameter, one more pass on the full sample to select variables. Thus, the solar computation load is less than $1/3$ that of lasso.

Given the computation loads for lasso and solar, we can work out the differences between bolasso and bsolar using the number of subsample repetitions (SR). Bolasso repeats lasso $256$ times while bolar-3 repeats solar only 3 times to obtain similar sparsity, bsolar-3 has approximately the same computation load as lasso.

\begin{table}[ht]
%
\centering
\caption{Simulation results for parallel computation time (mean runtime in seconds).\label{table:sim_load}}
\smallskip
\resizebox{0.98\textwidth}{!}{%
\renewcommand{\arraystretch}{0.7}
\begin{tabular}{l ... ... ...}
  \toprule
        & \multicolumn{3}{c}{$p/n\rightarrow0$}
        & \multicolumn{3}{c}{$p/n\rightarrow1$}
        & \multicolumn{3}{c}{$\log(p)/n\rightarrow0$} \\
        \cmidrule(lr){2-4} \cmidrule(lr){5-7} \cmidrule(lr){8-10}
        & \multicolumn{1}{c}{$\frac{100}{100}$} & \multicolumn{1}{c}{$\frac{100}{150}$} & \multicolumn{1}{c}{$\frac{100}{200}$} & \multicolumn{1}{c}{$\frac{150}{100}$} & \multicolumn{1}{c}{$\frac{200}{150}$} & \multicolumn{1}{c}{$\frac{250}{200}$} & \multicolumn{1}{c}{$\frac{400}{200}$} & \multicolumn{1}{c}{$\frac{800}{400}$} & \multicolumn{1}{c}{$\frac{1200}{600}$}  \\
  \midrule
  bsolar-3               & 0.05  & 0.07  & 0.08  & 0.06  & 0.08  & 0.12  & 0.32  & 0.51   & 1.04   \\
  bolasso (lars, 256 SR) & 9.52  & 12.49 & 10.61 & 10.01 & 13.92 & 19.72 & 23.10 & 184.59 & 502.56 \\
  bolasso (cd,   256 SR) & 13.49 & 60.51 & 60.35 & 13.92 & 16.85 & 20.17 & 27.73 & 100.58 & 308.12 \\
  \bottomrule
  \end{tabular}}
\end{table}

Table~\ref{table:sim_load} shows the average runtimes for the simulations. Generally speaking, bsolar-3 has a much shorter runtime than bolasso. When $n \times p$ is small (the first 5 columns), bsolar-3 runtime is roughly $0.5$-$1\%$ of bolasso runtime (assuming bolasso is solved by lars), which is consistent with the time complexity estimation. However, as $n$ and $p$ increase rapidly, parallel computation of bolasso is substantially more difficult to coordinate. This is primarily because our CPU must simultaneously generate 10-16 (the number of CPU cores) data matrices $X\in \mathbb{R}^{1200 \times 600}$, each of which must then be bootstrapped into 256 sub-matrices $X_{sub}\in \mathbb{R}^{1080 \times 600}$ for each CPU core to read again in parallel. This volume of data is more than our CPU can read from RAM in a single pass. As a result, the runtime differences are even more pronounced when $p$ and $n$ increase. The 256 subsample repetitions (totally 2816 lars or coordinate descent reptitions) render the bolasso selection algorithms computationally infeasible even with moderate $p$ and $n$. By contrast, bsolar-3 requires only 9 realizations of lars or coordinate descent. Due to a lighter computational load and CPU usage, bsolar-3 parallel computing is much easier to coordinate. As a result, the computation time difference will be much more substantial if the number of CPU cores is below $8$.

\subsubsection{Comparison with previous lasso computation research}

We thoroughly demonstrate the solar computation advantages in bootstrap selection in two steps. Firstly we show that, for given computation resources, our bolasso package almost attains the theoretical maximum speedup for lasso parallelization. Secondly, we show that bsolar is substantially faster than our bolasso package. Thus, given the computation resources, the speed of bsolar substantially exceeds the theoretical maximum speed of bolasso.

Given the same convergence criteria (tolerance for optimization and number of iterations), number of folds for CV ($K=10$), and number of $\lambda$s in the grid search (100), the time complexity of lasso is mostly determined by $n$, $p$, and pairwise correlations among the covariates ($corr$). For the purposes of comparison, we consider a Gaussian regression with $p/n=1000/100$ and $corr=0.5$.

\begin{itemize}
  \item With a 2.8GHz frequency, 2-core Intel Xeon CPU,
      \citet[Table 1]{friedman2010regularization} method reports an average runtime of 0.07 seconds for one pathwise coordinate descent realization (with covariance pre-computed for updating). The \citet{friedman2010regularization} package is coded in R with all numerical computations executed in Fortran/C++.
  \item Using an Intel Xeon W-3245 CPU with 3.2GHz frequency and 16 cores, the average runtime for the coordinate descent bolasso package is 41.92 seconds (with covariance pre-computed automatically), accounting for 256 realizations of 10-fold, cross-validated lasso (namely 2,816 pathwise coordinate descent realizations). Thus, the average runtime is 0.014 seconds per pathwise coordinate descent.
\end{itemize}

\noindent
Thus, with a similar CPU frequency and 14 additional cores, our lasso implementation produces an average speedup of $0.07/0.014=5.0$ times over \citet{friedman2010regularization} for each pathwise coordinate descent repetition.

Our code and the \citet{friedman2010regularization} code use the same design: 10 (parallelizable) pathwise coordinate descent repetitions to optimize $\lambda$ followed by a final (non-parallelizable) step to compute $\beta$. Roughly 11\% of the total computations (I/O, code interpretation to C++/Fortran, data generation, etc., matrix manipulation, and the step to compute $\beta$) are not parallelizable. Given $n$ and $p$, the maximum speedup according to Amdahl's law is:
\begin{equation}
  \frac{1}{\rho + (1-\rho)/s} = \frac{1}{0.11 + (1-0.11)/(16/2)} \approx 4.5,
\end{equation}
where $\rho$ is the proportion of computation that is not parallelizable and $s$ is the computation speedup for the parallelizable proportion (i.e., the core number multiple). Given that our CPU base frequency is also higher than \citet{friedman2010regularization} ($3.2$GHz over $2.8$GHz), we adjust the maximum speedup by the frequency multiple ($3.2/2.8$), resulting in a final maximum speedup of $4.5 \times 3.2/2.8 \approx 5.2$, or 4\% faster than our speedup of 5.0. Hence, given the core number and CPU frequency, our coordinate descent bolasso package achieves almost 96\% of the maximum possible speedup.

\begin{figure}[ht]
  \centering
  \includegraphics[width=0.8\linewidth]{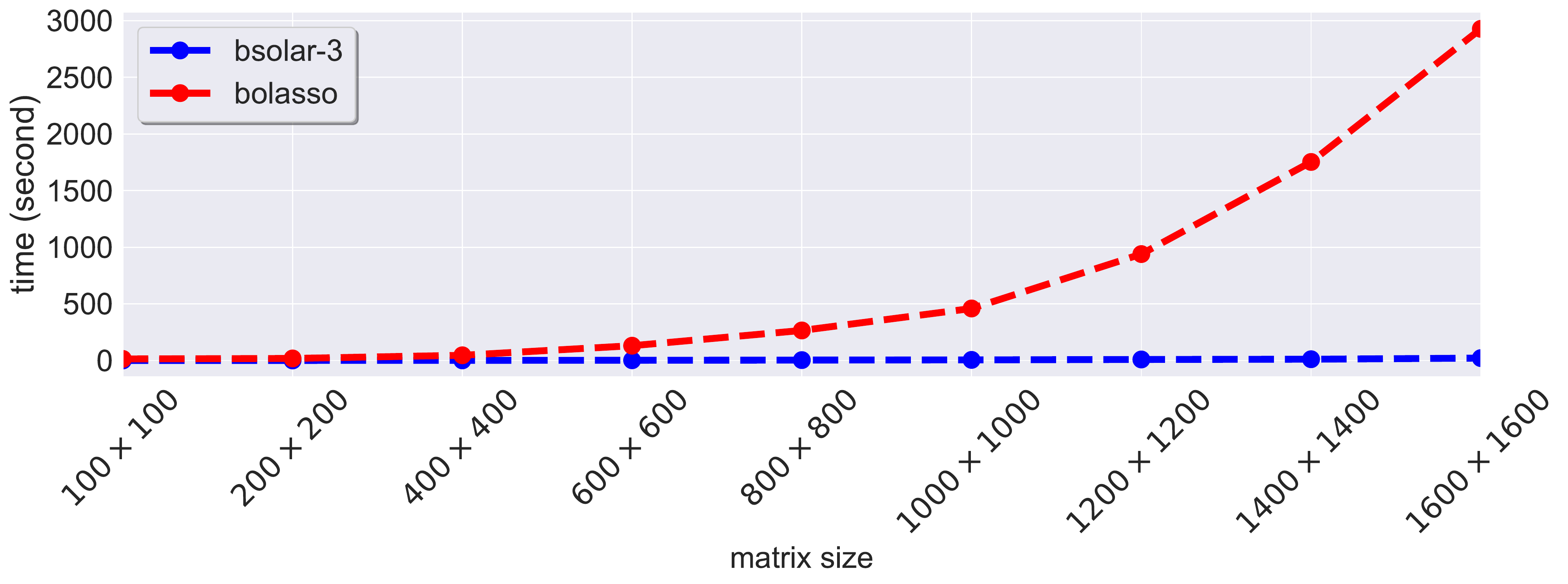}
  \caption{Average runtime (per pathwise coordinate descent) comparison for different $X$ matrix sizes.}
  \label{fig:runtime}
\end{figure}

As illustrated in Figure~\ref{fig:runtime}, bsolar easily outmatches the theoretical maximum speedup for paral- lelizing lasso-type estimators as n and p increases. Figure~\ref{fig:runtime} plots average runtime against the size ($n \times p$) of the $X$ matrix. As matrix size increases, the optimized bolasso package runtime rises exponentially while bsolar runtime increases linearly. Thus, bsolar easily outmatches the theoretical maximum speedup for bolasso as $n$ and $p$ increase, confirming the bsolar-3 advantage for high-dimensional data.

\subsubsection{Implication of solar computation advantages}

The efficiency of bsolar computation solves the issue of choosing a bootstrap variable selection threshold. \citet{bach2008bolasso} and \citet{meinshausen2010stability} claim that choosing a predefined value for the selection threshold ($f=1$, $0.9$, or $f\in\left[0.6,0.9\right]$) will return similarly sparse results. However, \citet{bach2008bolasso} and \citet{huang2014stat} show that predefined values may cause problems for variable selection. With $p/n=50/500$, $\mathrm{sd}(e)=3$, and true signal strength around $2$, \citet{huang2014stat} finds a 50\% false discovery rate with bootstrap selection methods, suggesting that the threshold still requires data-driven tuning. Moreover, the large number of bootstrap repetitions makes traditional parameter tuning methods (such as cross validation) computationally unaffordable for stability selection or bolasso. By contrast, the bsolar algorithm is efficient with even large $p,n$. The efficiency of bsolar means it is possible to tune $f$ by cross validation with a runtime of less than $6$ seconds for $p/n=1200/600$.

\section{Real-world data: Sydney house price prediction\label{section:application}}

To demonstrate that the improvements from solar are empirically feasible, we apply solar to real-world data. The real-world data reflect both the $p/n\rightarrow0$ scenarios as well as the challenging IRC settings, complicated dependence structures, and grouping effects typical of data in the social sciences.

The database is assembled from multiple sources. The primary source comprises real estate market transaction data for 11,974 Sydney, Australia, houses sold in 2010, including price and house attribute information (GIS coordinates, property address, bedrooms, bathrooms, car spaces, etc.). Each property is GIS-matched with: 2011 census data by Statistical Area Level 1 (the smallest census area in Australia, comprising at most 200 people or 60 households); 2010 and 2011 crime data by suburb; 2010 geo-spatial information on topology, climate, pollution, and aircraft noise; Google Maps data; 2009 primary and secondary school data; and 2010 Sydney traffic and public transport data (bus routes, train stations, and ferry wharfs). We predict house price with a linear model.

Using an ensemble of Bayes network learning algorithms for data cleaning, we reject variables with both very low conditional and unconditional correlations to house price. The remaining variables are listed in the first column of Table~\ref{table:house_variable}.\footnote{Due to the 200GB size of the database, we include only the data for these variables in the supplementary file.} The 57 variables fall into 5 broad categories: house attributes, distance to key locations (public transport, shopping, etc.), neighbourhood socioeconomic data, localized administrative and crime data, and local school quality. Pairwise correlations among all 57 covariates indicate, not surprisingly, severe multicollinearity and grouping effects, implying a harsh IRC setting.\footnote{Correlations and IRC are also reported in supplementary files.} Thus, heuristically increasing the value of the tuning parameter in lasso-type estimators (e.g., using the one-sd or the `elbow' rule) is unlikely to be useful since it may trigger further grouping effects and the random dropping of variables.

Table~\ref{table:house_variable} shows the selection comparison across the elastic net, lasso, and solar. With all variables in linear form, both lasso and elastic net lose sparsity, likely due to the complicated dependence structures and severe multicollinearity in the data, accordant with \citet{jia2010model}. By contrast, solar returns a much sparser model, with only $9$ variables selected from $57$. Very similar results are found with the variables in log form, hinting that solar possesses superior selection sparsity and robustness to a change in functional form. More importantly, solar variable selection outperforms the lasso-type estimators in terms of the balance between sparsity and prediction power. While pruning 25-48 variables from the elastic net and lasso selections, the post-selection regression $\mathrm{R}^2$ for solar falls by just 3-5\%.


\section{Conclusion}

In this paper we propose the solar (subsample-ordered least-angle regression) algorithm for high-dimensional data. Solar constructs solution paths using the $L_0$ norm and averages the resulting solution paths across subsamples, reducing sensitivity to high dimensionality while improving variable selection stability, efficiency, and accuracy. We prove that $L_0$ path averaging separates informative from redundant variables, that solar variable selection is consistent, and that the probability that solar omits weak signals is controllable for finite sample size.

Through simulations, examples, and real-world data, we demonstrate that, without any increase in computation load, solar yields substantial improvements over lasso in terms of the sparsity, stability, and accuracy of variable selection. We also find that solar largely avoids selection of redundant variables and rejection of informative variables in the presence of complicated dependence structures and harsh settings of the irrepresentable condition while conserving residual degrees of freedom for hypothesis testing. Relative to bootstrap lasso, bootstrapping solar improves selection sparsity and ranking accuracy and, for given computation resources, is substantially faster.

Detection of weak signals is a potential weakness in solar, although relative to lasso the difference is very slight. Nonetheless, we are working on an extension to solar, the double-bootstrap solar (DBsolar), which, if early results are any indication, promises to enable solar accurately to detect variables with weak signals.


\bibliographystyle{elsarticle-harv}

\bibliography{CVrefs}

\begin{table}[H]
  \centering
  \small
  \caption{Variable selection results for linear and log house price models. \label{table:house_variable}}
  \resizebox*{!}{0.95\textheight}{%
  \renewcommand{\arraystretch}{0.6}
  \begin{tabular}{@{}ll@{\extracolsep{6pt}}c@{\extracolsep{-2pt}}c@{\extracolsep{6pt}}c@{\extracolsep{-2pt}}c@{\extracolsep{6pt}}c@{\extracolsep{-2pt}}c@{}}
    \toprule
            &             & \multicolumn{2}{c}{elastic net}
                          & \multicolumn{2}{c}{lasso}
                          & \multicolumn{2}{c}{solar} \\
                          \cline{3-4} \cline{5-6} \cline{7-8} \\[-7pt]
    Variable & Description& \multicolumn{1}{c}{linear}
                          & \multicolumn{1}{c}{log}
                          & \multicolumn{1}{c}{linear}
                          & \multicolumn{1}{c}{log}
                          & \multicolumn{1}{c}{linear}
                          & \multicolumn{1}{c}{log} \\
    \midrule
    Bedrooms           & property, number of bedrooms             & \checkmark  & \checkmark  & \checkmark  & \checkmark  & \checkmark & \checkmark  \\
    Baths              & property, number of bathrooms            & \checkmark  & \checkmark  & \checkmark  & \checkmark  & \checkmark & \checkmark  \\
    Parking            & property, number of parking spaces       & \checkmark  & \checkmark  & \checkmark  & \checkmark  & \checkmark & \checkmark  \\
    AreaSize           & property, land size                      & \checkmark  & \checkmark  & \checkmark  & \checkmark  &   &    \\ \midrule
    Airport            & distance, nearest airport                & \checkmark  & \checkmark  & \checkmark  & \checkmark  &   &    \\
    Beach              & distance, nearest beach                  & \checkmark  & \checkmark  & \checkmark  & \checkmark  & \checkmark & \checkmark  \\
    Boundary           & distance, nearest suburb boundary        & \checkmark  & \checkmark  & \checkmark  & \checkmark  &   &    \\
    Cemetery           & distance, nearest cemetery               & \checkmark  &             & \checkmark  &    &   &    \\
    Child care         & distance, nearest child-care centre      & \checkmark  & \checkmark  & \checkmark  & \checkmark  &   & \checkmark  \\
    Club               & distance, nearest club                   & \checkmark  & \checkmark  & \checkmark  & \checkmark  &   &    \\
    Community facility & distance, nearest community facility     & \checkmark  & \checkmark  &    &    &   &    \\
    Gaol               & distance, nearest gaol                   & \checkmark  & \checkmark  &    &    & \checkmark & \checkmark  \\
    Golf course        & distance, nearest golf course            & \checkmark  & \checkmark  & \checkmark  & \checkmark  &   &    \\
    High               & distance, nearest high school            & \checkmark  & \checkmark  & \checkmark  & \checkmark  &   &    \\
    Hospital           & distance, nearest general hospital       & \checkmark  & \checkmark  &    & \checkmark  &   &    \\
    Library            & distance, nearest library                & \checkmark  &             & \checkmark  &    &   &    \\
    Medical            & distance, nearest medical centre         & \checkmark  & \checkmark  &    & \checkmark  &   &    \\
    Museum             & distance, nearest museum                 & \checkmark  & \checkmark  & \checkmark  & \checkmark  &   &    \\
    Park               & distance, nearest park                   & \checkmark  & \checkmark  & \checkmark  &    &   &    \\
    PO                 & distance, nearest post office            & \checkmark  & \checkmark  &    & \checkmark  &   &    \\
    Police             & distance, nearest police station         & \checkmark  & \checkmark  & \checkmark  & \checkmark  &   &    \\
    Pre-school         & distance, nearest preschool              & \checkmark  & \checkmark  & \checkmark  & \checkmark  &   &    \\
    Primary            & distance, nearest primary school         & \checkmark  & \checkmark  & \checkmark  & \checkmark  &   &    \\
    Primary High       & distance, nearest primary-high school    & \checkmark  & \checkmark  & \checkmark  & \checkmark  &   &    \\
    Rubbish            & distance, nearest rubbish incinerator    & \checkmark  & \checkmark  & \checkmark  &    &   &    \\
    Sewage             & distance, nearest sewage treatment       & \checkmark  &    &    &    &   &    \\
    SportsCenter       & distance, nearest sports centre          & \checkmark  & \checkmark  & \checkmark  & \checkmark  &   &    \\
    SportsCourtField   & distance, nearest sports court/field     & \checkmark  &             & \checkmark  & \checkmark  &   &    \\
    Station            & distance, nearest train station          & \checkmark  &             & \checkmark  &    &   &    \\
    Swimming           & distance, nearest swimming pool          & \checkmark  & \checkmark  & \checkmark  & \checkmark  &   &    \\
    Tertiary           & distance, nearest tertiary school        & \checkmark  & \checkmark  & \checkmark  & \checkmark  &   &    \\
    \midrule
    Mortgage           & SA1, mean mortgage repayment (log)       & \checkmark  & \checkmark  & \checkmark  & \checkmark  & \checkmark & \checkmark  \\
    Rent               & SA1, mean rent (log)                     & \checkmark  & \checkmark  & \checkmark  & \checkmark  & \checkmark & \checkmark  \\
    Income             & SA1, mean family income (log)            & \checkmark  & \checkmark  & \checkmark  & \checkmark  & \checkmark & \checkmark  \\
    Income (personal)  & SA1, mean personal income (log)          & \checkmark  &    &    &    &   &    \\
    Household size     & SA1, mean household size                 & \checkmark  & \checkmark  & \checkmark  & \checkmark  &   &    \\
    Household density  & SA1, mean persons to bedroom ratio       & \checkmark  & \checkmark  & \checkmark  & \checkmark  &   &    \\
    Age                & SA1, mean age                            & \checkmark  & \checkmark  & \checkmark  & \checkmark  &   & \checkmark  \\
    English spoken     & SA1, percent English at home             & \checkmark  &             & \checkmark  &    &   &    \\
    Australian born    & SA1, percent Australian-born             & \checkmark  &             & \checkmark  &    &   &    \\
    \midrule
    Suburb area        & suburb area                             & \checkmark  &    & \checkmark  & \checkmark  &   &    \\
    Population         & suburb population                       & \checkmark  & \checkmark  &    & \checkmark  &   &    \\
    TVO2010            & suburb total violent offences, 2010     & \checkmark  &             &    &    &   &    \\
    TPO2010            & suburb total property offences, 2010    & \checkmark  & \checkmark  &    & \checkmark  &   &    \\
    TVO2009            & suburb total violent offences, 2009     & \checkmark  & \checkmark  & \checkmark  &    &   &    \\
    TPO2009            & suburb total property offences, 2009    & \checkmark  & \checkmark  &    &    &   &    \\
    \midrule
    ICSEA              & local school, socio-educational advantage & \checkmark  & \checkmark  & \checkmark  & \checkmark  & \checkmark & \checkmark  \\
    ReadingY3          & local school, year 3 mean reading score  & \checkmark  & \checkmark  & \checkmark  & \checkmark  &   &    \\
    WritingY3          & local school, year 3 mean writing score  & \checkmark  & \checkmark  & \checkmark  & \checkmark  &   &    \\
    SpellingY3         & local school, year 3 mean spelling score & \checkmark  & \checkmark  & \checkmark  &    &   &    \\
    GrammarY3          & local school, year 3 mean grammar score  & \checkmark  &             & \checkmark  &    &   &    \\
    NumeracyY3         & local school, year 3 mean numeracy score & \checkmark  & \checkmark  & \checkmark  & \checkmark  &   &    \\
    ReadingY5          & local school, year 5 mean reading score  & \checkmark  &    &    &    &   &    \\
    WritingY5          & local school, year 5 mean writing score  & \checkmark  & \checkmark  & \checkmark  &    &   &    \\
    SpellingY5         & local school, year 5 mean spelling score & \checkmark  & \checkmark  & \checkmark  &    &   &    \\
    GrammarY5          & local school, year 5 mean grammar score  & \checkmark  & \checkmark  & \checkmark  &    &   &    \\
    NumeracyY5         & local school, year 5 mean numeracy score & \checkmark  &             &    &    &   &    \\
    \midrule
                      & Number of variables selected  & 57 & 45 & 44 & 36 & 9 & 11 \\
                      & post-selection OLS $R^2$      & 0.55 & 0.76 & 0.55 & 0.76 & 0.50 & 0.73\\
                      & Sample size & \multicolumn{6}{c}{11,974} \\

    \bottomrule

  \end{tabular}}

\end{table}

\clearpage
\appendixpagenumbering

\begin{appendices}
\section{$L_0$ path ranking accuracy and variable selection consistency}

For generality, we derive the theoretical properties of solar under the general forward selection framework of \citet[Figure~1]{zhang09}. Our proof method is summarized as follows. Under various settings and assumptions, \citet{tropp2004greed}, \citet{yuan2007non}, \citet{wainwright2009sharp}, \citet{zhang09}, and \citet{ing2011stepwise} have shown: (i) forward selection is consistent in different modes, and (ii) informative variables are ranked higher at earlier stages of the solution path than redundant variables. Since the $L_0$ path on each solar subsample is essentially the re-parameterized forward selection path, (i) and (ii) can be applied directly to the $L_0$ path on each solar subsample. As a result, we can build a probabilistic lower bound to show that, on average, (i) and (ii) also hold for the average $L_0$ path and for the variable selection result on the average $L_0$ path (including solar variable selection). We follow the approach of \citet{tropp2004greed}, \citet{wainwright2009sharp}, and \citet{zhang09}, because their assumptions and methods of analysis are similar to the theoretical analysis of lasso-type estimators. We follow the \citet{zhang09} notation.

\begin{definition}
  Consider the regression model $Y = X\beta + \mathbf{e}$, \label{def:notation}
\end{definition}
\begin{enumerate}
  \item     $Y = \left[ y_1, \ldots, y_n \right]^T \in \mathbb{R}^{n \times 1}$ is the response variable.

  \item     The data matrix is $X = \left[ \mathbf{x}_1, \ldots,\mathbf{x}_p \right] \in \mathbb{R}^{n \times p}$ with columns $\mathbf{x}_j \in \mathbb{R}^{n \times 1}$, $\forall j=1,\ldots,p$, and rows $ X_{i,\bigcdot} \in \mathbb{R}^{1 \times p}$, $\forall i = 1,\ldots,n$. $\mathbf{e}\in \mathbb{R}^{n\times 1}$ is a stochastic noise term.
  \item     The regression coefficients of the data generating process (DGP) are $\overline{\beta} = \left[\,\overline{\beta}_1, \ldots, \overline{\beta}_{p_1}, \mathbf{0} \right]^T \in \mathbb{R}^{p\times 1}$, where the first $p_1$ entries are not $0$.
  \item     The support, $\forall \beta \in \mathbb{R}^{p \times 1}$, is $\mathrm{supp}(\beta) = \{ j : \beta_j \neq 0 \} $.
  \item     Given $X \in \mathbb{R}^{n \times p}$ and $F \subset \left\{1, \ldots, p \right\}$,
  \begin{displaymath}
    \widehat{\beta}_{X} \left( F,Y \right) =
    \mathrm{argmin}_{ \beta \in \mathbb{R}^{p \times 1}} \;\;
    \frac{1}{n}\left\Vert X\beta - Y\right\Vert_{2}^{2}
    \quad\mbox{subject to}\quad \mathrm{supp}(\beta)\subset F.
  \end{displaymath}
  That is, $\widehat{\beta}_{X}\left(F,Y\right)$ is the least squares solution with coefficients restricted to $F$.
  \item $\left\vert F \right\vert$ is the cardinality of $F$ while $\overline{F}-F$ is the difference of sets $\overline{F}$ and $F$.
  \item $X_F$ is an $n \times \left\vert F \right\vert$ matrix with columns $\mathbf{x}_j \in \left[ \mathbf{x}_0, \ldots \mathbf{x}_p \right]$ with $j\in F$ arranged in ascending order.
  \item To introduce the irrepresentable condition and sparse eigenvalue condition, define
  \begin{displaymath}
    \mu_{X}\left(F\right)=\max_{j\in F}\left\Vert \left(X_{F}^{T}X_{F}\right)^{-1}X_{F}^{T}\mathbf{x}_{j}\right\Vert _{1}
  \end{displaymath}
  and
  \begin{displaymath}
    \rho_{X}\left(F\right)=\inf\left\{ \frac{1}{n}\left\Vert X\beta\right\Vert _{2}^{2}/\left\Vert \beta\right\Vert _{2}^{2}:\mathrm{supp}\left(\beta\right)\subset F\right\}
  \end{displaymath}
  \item (The $\epsilon$ stopping rule.) The \citet[Figure~1]{zhang09} framework is also known as orthogonal matching pursuit (OMP). Prior to stage $l$, forward selection finds the unselected variable
  \begin{displaymath}
    \mathbf{x}^{(l)} = \mathrm{argmax}_{\mathbf{x}_j} \;\; \left\vert \mathbf{x}_j^T u^{(l-1)} \right\vert, \\
    \quad\mbox{for all unselected }\mathbf{x}_{j},
  \end{displaymath}
  where $u^{\left(l-1\right)}$ is the forward regression residual of stage $l-1$. If $\left\vert (\mathbf{x}^{(l)})^T u^{(l-1)} \right\vert > \epsilon$, the forward selection loop will select $\mathbf{x}^{(l)}$, compute $u^{(l)}$ and move to stage $l+1$; otherwise, the forward selection loop will stop and report the regression coefficients and selected variables on and before stage $l-1$.
\end{enumerate}

We also adopt the \citet{zhang09} assumptions:
\begin{description}
  \item [{[}A1{]}] each $\mathbf{x}_i$ is normalized such that $\left\Vert \mathbf{x}_{j}\right\Vert _{2}^{2}/n=1,\;\forall j=1,\ldots,p$;
  \item [{[}A2{]}] the $\overline{\beta}$ is sparse: $\exists \overline{\beta} \in \mathbb{R}^{p\times 1}$ with $\overline{F} = \mathrm{supp} \left( \overline{\beta} \right)$ such that $ \mathbf{E} \left( Y \right) = X\overline{\beta} = \left[ X_{1,\bigcdot} \; \overline{\beta}, \;  \ldots, \; X_{n,\bigcdot} \; \overline{\beta} \right]^T$;
  \item [{[}A3{]}] the irrepresentable condition \citep{tropp2004greed}: $\mu_{X}\left(\overline{F}\right)<1$; the sparse eigenvalue condition \citep{wainwright2009sharp}: $\rho_{X}\left(\overline{F}\right)>0$.
  \item [{[}A4{]}] there exists a $ \sigma > 0$ such that $Y = \left[ Y_1, \ldots, Y_n \right]$ are independent (but not necessarily identically distributed) sub-Gaussians with $\mathbb{E} \left(Y_i\right)=  X_{i,\bigcdot} \; \overline{\beta}$ and $\mathbb{E}_{Y_i} \left( e^{t\left(Y_i - \mathbb{E} \left( Y_i \right) \right)}\right) \leqslant e^{\sigma^2t^2/2}$, $\forall t\in \mathbb{R}$ and $\forall i\in \left\{1,\ldots,n\right\}$.
\end{description}

\textbf{A4} implies that $Y$ can be either bounded or unbounded. \textbf{A2} and \textbf{A4} imply that the regression noise in the DGP, $\left[e_1, \ldots, e_n\right]^T$, are independent sub-Gaussians.

The proof of variable selection accuracy and variable ranking accuracy is specified in the following steps.


\subsection*{Step 1 : re-parameterize the $\epsilon$ stopping rule via $\widehat{q}^k$ }

\citet[Theorem 1]{zhang09} shows that the probability of omitting informative variables is bounded on a finite sample.
\begin{theorem}
  \label{theoremA1}
  \citep{zhang09} Consider the forward selection algorithm with Assumption 1 satisfied. Given any $\eta \in \left( 0, 1 \right)$, with probability larger than $1 - \eta$, if the $\epsilon$  stopping criterion stops forward selection at stage $l$, satisfying
  \begin{equation}
    \epsilon>\frac{1}{1-\mu_X \left( \overline{F} \right)} \sigma\sqrt{ 2 \ln \left( 4p / \eta \right)}
  \end{equation}
  and
  \begin{equation}
    \min_{j\in\overline{F}}\left|\overline{\beta}_{j}\right|\geqslant\frac{3\epsilon}{\rho_{X}\left(\overline{F}\right)\cdot\sqrt{n}}, \label{thm0}
  \end{equation}
  then when the procedure stops at stage $l$,
  \begin{equation}
    \overline{F}=F^{\left(l-1\right)}, \nonumber
  \end{equation}
  where $F^{\left(l-1\right)}$ is the set of variable selected at stage $l-1$. $\blacksquare$
\end{theorem}

As shown in Definition~\ref{def:notation}, \citet{zhang09} executes each forward selection stage using $\epsilon$. By contrast, we re-parameterize forward selection stages (in Algorithm 1) and the stopping rule (in Algorithm 2) using $\widehat{q}$. Given the assumption $\epsilon > \frac{1}{1-\mu_X \left( \overline{F} \right)} \sigma\sqrt{ 2 \ln \left( 4p / \eta \right)}$, we can find an equivalent stopping criterion based on $\widehat{q}$ as follows,

\begin{itemize}
  \item Assume that, on subsample $\left(Y^{k},X^{k}\right)$, $\epsilon>\frac{1}{1-\mu_{X}\left(F\right)}\sigma\sqrt{2\ln\left(4p/\eta\right)}$ stops the forward selection at stage $l^{k}$;
  \item from line 6 of Algorithm 1, variables selected before stage $l^{k}$ must have $\widehat{q}_i^k$ values $> \left(\widetilde{p} + 1 - l^k \right) / \widetilde{p}$, where $\widetilde{p} = \min\left\{ n(K-1)/K, p\right\}$
  \item hence, on the $k$\textsuperscript{th} subsample, the stopping rule
  \begin{displaymath}
    \mbox{forward selection stops at stage } l^k
  \end{displaymath}
  is equivalent to the stopping rule
  \begin{displaymath}
    \mbox{forward selection only selects the variables } \left\{ \mathbf{x}_j:\widehat{q}_j^k>\left(\widetilde{p}+1-l^k\right)/\widetilde{p}\right\}.
  \end{displaymath}
\end{itemize}
\medskip

\citet{zhang09} also assumes that, when forward selection stops at stage $l^k$ on subsample $\left(Y^{k},X^{k}\right)$,
\begin{displaymath}
  \epsilon>\frac{1}{1-\mu_{X}\left(F\right)}\sigma\sqrt{2\ln\left(4p/\eta\right)},
  \;\exists\; \eta \in \left(0, 1\right).
\end{displaymath}
The assumption plays a key role for true signal recovery. Since solar does not explicitly use $\epsilon$, we need to re-parameterize the assumption before applying the \citet{zhang09} result. Specifically, denote
\begin{equation}
  \omega^{\left(l^k\right)} = \left\vert \left( \mathbf{x}^{\left( l^k \right)}\right) ^T u^{\left(l^k \right)} \right\vert,
\end{equation}
where $u^{\left( l^k \right)}$ is the regression residual and $\mathbf{x}^{ \left( l^k \right)}$ the variable selected at stage $l^k$. Solar selects variables on the $L_0$ and average $L_0$ paths based on $\widehat{q}^k$ and $\widehat{q}$. Hence, we can analyze the selection decisions of solar by examining the absolute co-movement between $u^{\left( l^k \right)}$ and $\mathbf{x}^{ \left( l^k \right)}$ on the $L_0$ path, which is identical to the $\epsilon$ assumption. This can be shown as follows.
\begin{itemize}
  \item Assume the $\epsilon$ stopping rule stops forward selection at stage $l^k$ and $\epsilon > \frac{1}{1-\mu_{X}\left(F\right)}\sigma\sqrt{2\ln\left(4p/\eta\right)}$. We must have
  \begin{equation}
    \omega^{\left(l^k - 1\right)} > \epsilon > \frac{1}{1 - \mu_X \left( F \right)} \sigma \sqrt{ 2 \ln \left(4p / \eta \right)}
  \end{equation}
  Hence, stopping forward selection based on $\epsilon > \frac{1}{1 - \mu_X \left( F \right)} \sigma \sqrt{ 2 \ln \left(4p / \eta \right)}$ implies
  \begin{displaymath}
    \omega^{\left(l^k - 1\right)} > \frac{1}{1 - \mu_X \left( F \right)} \sigma \sqrt{ 2 \ln \left(4p / \eta \right)}.
  \end{displaymath}
  \item Assume $\omega^{\left(\bigcdot\right)} > \frac{1}{1 - \mu_X \left( F \right)} \sigma \sqrt{ 2 \ln \left(4p / \eta \right)}$ is not violated until stage $l^k$. This implies that we have $\omega^{\left(l^k - 1 \right)} > \frac{1}{1 - \mu_X \left( F \right)} \sigma \sqrt{ 2 \ln \left(4p / \eta \right)}$. Besides, since $\omega^{\left(\bigcdot\right)} > \frac{1}{1 - \mu_X \left( F \right)} \sigma \sqrt{ 2 \ln \left(4p / \eta \right)}$ for the first $l^k-1$ stages, there exists some $\epsilon^*$ such that
  \begin{equation}
    \epsilon^* \in \left( \frac{1}{1 - \mu_X \left( F \right)} \sigma \sqrt{ 2 \ln \left(4p / \eta \right)} \; , \; \min_{i < l^k} \; \omega^{\left( i \right)} \right) \label{eqn:omega_epsilon}
  \end{equation}
  If we equip the $\epsilon$ stopping rule with the $\epsilon^*$, the $\epsilon$ stopping rule will also stop forward selection at stage $l^k$.
\end{itemize}
\medskip

\noindent
Hence, `the $\epsilon > \frac{1}{1 - \mu_X \left( F \right)} \sigma \sqrt{ 2 \ln \left(4p / \eta \right)}$ stopping rule stops forward selection at stage $l^k$' is equivalent to `the last variable that forward selection selects before stop has its $\omega$ larger than $\frac{1}{1-\mu_{X}\left(F\right)}\sigma\sqrt{2\ln\left(4p/\eta\right)}$'. As such, we can re-parameterize the stopping rule on $\left(Y^{k},X^{k}\right)$ based on $\widehat{q}_{i}^{k}$ as follows
\begin{definition}
  (The $\epsilon$ and $\widehat{q}^{\,k}$ stopping assumptions) \label{def:stopping assumption}
\end{definition}

\begin{itemize}
  \item We refer to the \citet{zhang09} assumption on $\epsilon$ stopping rule as \emph{the $\epsilon$ stopping assumption}:
  \begin{eqnarray}
    \mbox{when forward selection stops at stage } l^k \mbox{ on subsample } \left(Y^k, X^k\right), \\ \notag
    \exists \eta \in \left( 0, 1\right) \mbox{ such that }\epsilon>\frac{1}{1-\mu_{X}\left(F\right)}\sigma\sqrt{2\ln\left(4p/\eta\right)}.
  \end{eqnarray}
  \item Assume forward selection only selects $\left\{ \mathbf{x}_j:\widehat{q}_j^k > \frac{\widetilde{p} + 1 - l^k} {\widetilde{p}} \right\}$ on subsample $\left(Y^k, X^k \right)$. We define
  \begin{displaymath}
    \omega^{\left(l^k\right)} = \left\vert \left( \mathbf{x}^{\left( l^k \right)}\right) ^T u^{\left(l^k \right)} \right\vert.
  \end{displaymath}
  For simplicity, we denote $\omega$ as $\omega^{\left( \bigcdot \right)}$ of the last variable that forward selection selects.
  \item we refer to the following equivalent assumption for solar as \emph{the $\widehat{q}^{\,k}$ stopping assumption}:
  \begin{align}
    \mbox{when forward selection only selects } \left\{ \mathbf{x}_j:\widehat{q}_j^k > \frac{\widetilde{p} + 1 - l^k} {\widetilde{p}} \right\} \mbox{ on } \left(Y^k, X^k \right), \notag \\
    \exists\, \eta \in \left( 0, 1\right) \mbox{ such that } \omega >\frac{1}{1-\mu_{X}\left(F\right)}\sigma\sqrt{2\ln\left(4p/\eta\right)}. \label{eqn:qk_stopping}
  \end{align}
\end{itemize}
\medskip

Based on Definition~\ref{def:stopping assumption}, we can re-parameterize Theorem~\ref{theoremA1} into Lemma~\ref{lemma:1}.

\begin{lemma}
  Consider the forward selection algorithm on the $k$th subsample $\left(Y^{k},X^{k}\right)$ with Assumption 1 satisfied. With probability larger than $1-\eta$, if (\ref{eqn:qk_stopping}) is satisfied and
  \begin{displaymath}
      \min_{j\in\overline{F}}\left|\overline{\beta}_{j}\right|\geqslant\frac{3\omega}{\rho_{X}\left(\overline{F}\right)\cdot\sqrt{n\left(K-1\right)/K}},
  \end{displaymath}
  then
  \begin{displaymath}
      \overline{F}=\left\{ \mathbf{x}_{j}:\widehat{q}_{i}^{k}>c^k\right\}.
  \end{displaymath}
  where $c^k = \left(\widetilde{p} + 1 - l^k\right) / \widetilde{p}$.
  \label{lemma:1}
\end{lemma}

\begin{proof}
  Lemma~\ref{lemma:1} is derived by replacing the $\epsilon$ stopping assumption with the $\widehat{q}^k$ stopping assumption. Note that
  \begin{displaymath}
    \min_{j\in\overline{F}}\left|\overline{\beta}_{j}\right|\geqslant\frac{3\omega}{\rho_{X}\left(\overline{F}\right)\cdot\sqrt{n\left(K-1\right)/K}} \implies \min_{j\in\overline{F}}\left|\overline{\beta}_{j}\right| \geqslant \frac{3\epsilon}{\rho_{X}\left(\overline{F}\right)\cdot\sqrt{n\left(K-1\right)/K}},
\end{displaymath}
since (\ref{eqn:omega_epsilon}) implies that $\exists \epsilon > \frac{1}{1-\mu_{X}\left(F\right)}\sigma\sqrt{2\ln\left(4p/\eta\right)}$ such that $\omega > \epsilon$.  We also replace $n$ in (\ref{thm0}) with $n\left(K-1\right)/K$ since each subsample randomly drops $1/K$ of the original sample points in Algorithm 1.
\end{proof}
\medskip


\subsection*{Step 2 : averaging the solution paths}

Since we assume the nonzero $\overline{\beta}_i$ are the first $p_1$ components of $\overline{\beta}$, x can be rewritten as
\begin{lemma}
  Consider the forward selection algorithm on the $k$th subsample $\left(Y^{k},X^{k}\right)$ with Assumption 1 satisfied. With probability less than $\eta$, if (\ref{eqn:qk_stopping}) is satisfied and
  \[
      \min_{j\in\overline{F}}\left|\overline{\beta}_{j}\right|\geqslant\frac{3\omega}{\rho_{X}\left(\overline{F}\right)\cdot\sqrt{n\left(K-1\right)/K}},
  \]
  then
  \[
  \begin{cases}
      \widehat{q}_j^k \leqslant c^k, \; \forall j\leqslant p_1\\
      \widehat{q}_j^k > c^k, \; \forall j>p_1.
  \end{cases}
  \]
  where $c^k = \left(\widetilde{p} + 1 - l^k\right) / \widetilde{p}$.
  \label{lemma:2}
\end{lemma}
\noindent
Lemma~\ref{lemma:2} directly implies that, with high probability, you can find a threshold value $c^k$ on the $L_0$ path that perfectly separates the informative from the redundant variables.

To accommodate multiple subsamples in the average $L_0$ path, we define the $\widehat{q}$ stopping rule by slightly modifying (\ref{eqn:qk_stopping}).
\begin{definition}
  (The assumption for the $\widehat{q}$ stopping rule). We refer to the following rule as \emph{the $\widehat{q}$ stopping assumption} for the average $L_0$ path:
  \begin{eqnarray}
    \mbox{when forward selection only selects } \left\{ \mathbf{x}_j:\widehat{q}_j^k > \left( \widetilde{p} + 1 - l^k \right) / \widetilde{p} \right\} \mbox{ on } \left(Y^k, X^k \right), \notag \\
    \exists\, \eta \in \left( 0, 1/K\right) \mbox{ such that } \omega >\frac{1}{1-\mu_{X}\left(F\right)}\sigma\sqrt{2\ln\left( \frac{4p}{K\eta}\right)}.
    \label{eqn:q_hat_stopping}
  \end{eqnarray}
\end{definition}
\medskip

Lemma~\ref{lemma:3} follows from the $\widehat{q}$ stopping assumption.

\begin{lemma}
  Consider the forward selection algorithm on the average $L_{0}$ path with Assumption 1 satisfied. With probability less than $\eta$, if (\ref{eqn:q_hat_stopping}) is satisfied and
  \[
      \min_{j\in\overline{F}}\left|\overline{\beta}_{j}\right|\geqslant\frac{3\omega}{\rho_{X}\left(\overline{F}\right)\cdot\sqrt{n\left(K-1\right)/K}},
  \]
  then
  \[
  \begin{cases}
      \frac{1}{K}\sum\widehat{q}_{i}^{k}=\widehat{q}_{i}> c^* ,\forall i\leqslant p_{1}\\
      \frac{1}{K}\sum\widehat{q}_{i}^{k}=\widehat{q}_{i}\leqslant c^*,\forall i>p_{1}
  \end{cases}
  \]
  where $c^* = \frac{1}{K} \sum_k^K c^k / K$ and $c^k = \left(\widetilde{p} + 1 - l^k\right) / \widetilde{p}$.
  \label{lemma:3}
\end{lemma}

\begin{proof}
    The proof is a direct result from Lemma~2. If we apply the $c$ stopping rule, Lemma~2 implies that
    \begin{eqnarray}
        Pr\left\{ \widehat{q}_i^k \leqslant c^k,\forall i\leqslant p_1 \mbox{ and } \widehat{q}_j^k > c^k,\forall j>p_1 \right\} \leqslant \eta/K.
    \end{eqnarray}
    Since, for multiple events $A_i$, $Pr\left\{ \cap_i A_i\right\} \leqslant \sum_i Pr\left\{ A_i\right\}$, we have
    \begin{eqnarray}
        Pr\left\{ \sum_{k=1}^K \widehat{q}_i^k \leqslant \sum_{k=1}^K c^k,\forall i\leqslant p_1 \mbox{ and } \sum_{k=1}^K  \widehat{q}_j^k > \sum_{k=1}^K c^k,\forall j>p_1 \right\} \leqslant \eta.
    \end{eqnarray}
    Since, $\widehat{q}_i = \frac{1}{K}\sum\widehat{q}_{i}^{k}$ and $c^{*} = \frac{1}{K}\sum c^k$, we have
    \begin{eqnarray}
        Pr\left\{ \widehat{q}_i \leqslant c^{\,*},\forall i \leqslant p_1 \mbox{ and } \widehat{q}_j > c^{*},\forall j>p_1 \right\} \leqslant \eta.
    \end{eqnarray}
\end{proof}
\noindent
Lemma~\ref{lemma:3} irectly implies that, with high probability, you can find a threshold value $c^*$ on the average $L_0$ path that perfectly separates the informative from the redundant variables.


\subsection*{Step 3 : variable selection consistency}

Based on Lemma~\ref{lemma:3}, Theorem~\ref{thm:1} on variable selection consistency is straightforward.

\begin{theorem}
  Consider the forward selection algorithm on the average $L_0$ path with Assumption 1 satisfied, noise $\sigma$ independent of $n$. Assume that the strong irrepresentable condition holds. For each sample size $n$, denote $F\left( n \right)$ as the index set of selected variables when forward selection stops with $\omega\geqslant n^{s/2}$, $\forall s\in(0,1]$ , and $\overline{F}\left(n\right)$ as the corresponding index set of informative variables. We have
  \[
      Pr\left(\;F\left(n\right)\neq\overline{F}\left(n\right)\;\right)\leqslant \exp\left(-\frac{n^{s}}{\log\left(n\right)}\right)
  \]
  if
  \[
      p\left(n\right)\leqslant\exp\left(\frac{n^{s}}{\log\left(n\right)}\right),
  \]
  and
  \[
      \min_{j\in\overline{F}}\left|\overline{\beta}_{j}\right|\geqslant\frac{3n^{(s-1)/2}}{\rho_{X}\left(\overline{F}\left(n\right)\right)}
  \]
  where $p\left(n\right)$ is the total dimension of variable as $n$ increases.
  \label{thm:1}
\end{theorem}

\begin{proof}
  When $n$ is sufficiently large, the assumptions
  \[
      \omega^{\left(l^k \right)} = \left\vert \left( \mathbf{x}^{\left( l^k \right)}\right) ^T u^{\left(l^{k}\right)} \right\vert >\frac{1}{1-\mu_{X}\left(F\right)}\sigma\sqrt{2\ln\left(4p/\eta\right)}
  \]
  and
  \[
      \min_{j\in\overline{F}}\left|\overline{\beta}_{j}\right|\geqslant\frac{3\omega}{\rho_{X}\left(\overline{F}\right)\cdot\sqrt{n\left(K-1\right)/K}}
  \]
  hold with $\eta = \exp\left(-n^s/\log(n)\right)$. Thus, Theorem~\ref{thm:1} follows from Lemmas~\ref{lemma:2} and \ref{lemma:3}.
\end{proof}


\subsection*{Step 4 : probability of omitting weak signals}

\citet[Theorem 2]{zhang09} shows that, if
\begin{displaymath}
  \min_{j\in\overline{F}}\left|\overline{\beta}_{j}\right| < \frac{3\omega}{\rho_{X}\left(\overline{F}\right)\cdot\sqrt{n\left(K-1\right)/K}},
\end{displaymath}
the probability of selecting at least one redundant $\mathbf{x}_i$ and the number of omitted weak signals are still bounded by sample size, the sparse eigenvalue condition, and the stopping condition. Using the same procedure as step 1, we can rewrite \citet[Theorem 2]{zhang09} into Lemma~\ref{lemma:4}.

\begin{lemma}
  \label{lemma:4}
  Consider the forward selection algorithm on the $k$th subsample $\left(Y^{k},X^{k}\right)$ with Assumption 1 satisfied. With probability larger than $1-\eta$, if
  \begin{equation}
    \omega > \frac{1}{1-\mu_X \left( \overline{F} \right)} \sigma\sqrt{ 2 \ln \left( 4p / \eta \right)}
  \end{equation}
  then when the procedure stops at stage $l$, the following claims are true:
  \[
  \begin{cases}
    \left\{ \mathbf{x}_{j}:\widehat{q}_{i}^{k}>c^k\right\} \subset \overline{F}\\
    \left\vert \left\{ \mathbf{x}_{j}:\widehat{q}_{i}^{k}>c^k\right\} -  \overline{F} \right\vert \leqslant 2 \left\vert \left\{ j \in \overline{F} : \left\vert \overline{\beta}_j < \frac{3 \omega}{\rho_{X}\left(\overline{F}\right) \cdot \sqrt{n\left(K-1\right)/K}} \right\vert \}\right\} \right\vert
  \end{cases}
  \]
  where $F^{\left(l-1\right)}$ is the set of variable selected at stage $l-1$. $\blacksquare$
\end{lemma}

Using the same method as Lemma~\ref{lemma:3}, Lemma~\ref{lemma:5} follows from the $\widehat{q}$ stopping assumption, showing that selection errors are strictly restricted by sample size, the sparse eigenvalue condition, and the stopping condition on the average path.

\begin{lemma}
  Consider the forward selection algorithm on the average $L_{0}$ path with Assumption 1 satisfied. With probability less than $\eta$, if (\ref{eqn:q_hat_stopping}) is satisfied, then
  \[
  \begin{cases}
    \left\{ \mathbf{x}_{j}:\widehat{q}^{k}>c^*\right\} \subset \overline{F}\\
    \left\vert \left\{ \mathbf{x}_{j}:\widehat{q}^{k}>c^*\right\} -  \overline{F} \right\vert > 2 \left\vert \left\{ j \in \overline{F} : \left\vert \overline{\beta}_j < \frac{3 \omega}{\rho_{X}\left(\overline{F}\right) \cdot \sqrt{n\left(K-1\right)/K}} \right\vert \}\right\} \right\vert
  \end{cases}
  \]
  where $c^* = \frac{1}{K} \sum_k^K c^k / K$ and $c^k = \left(\widetilde{p} + 1 - l^k\right) / \widetilde{p}$.
  \label{lemma:5}
\end{lemma}
\clearpage
\end{appendices}

\end{document}